\documentclass{article}
\usepackage[accepted]{icml2020}
\usepackage[utf8]{inputenc} 
\usepackage[T1]{fontenc}    
\usepackage[draft]{hyperref}       
\usepackage{url}            
\usepackage{booktabs}       
\usepackage{amsfonts}       
\usepackage{nicefrac}       
\usepackage{microtype}      
\usepackage{amsmath}
\usepackage{amsthm}
\usepackage{subcaption}
\usepackage{flushend}
\usepackage{times}
\usepackage{graphicx}
\usepackage{wrapfig,lipsum}
\usepackage{enumitem}
\usepackage{amssymb}
\usepackage{wrapfig}
\setlist{nolistsep,leftmargin=*}
\usepackage{xcolor}
\usepackage{verbatim}
\usepackage{physics} 
\usepackage[compact]{titlesec}
\titlespacing{\section}{0pt}{*0}{*0}
\titlespacing{\subsection}{0pt}{*0}{*0}
\titlespacing{\subsubsection}{0pt}{*0}{*0}
\usepackage{thmtools,thm-restate}
\usepackage{adjustbox}
\newcommand{\newtext}[1]{{\textcolor{black}{#1}}}
\newcommand{\app}{Suppl. Section~}
\newtheorem{theorem}{Theorem}
\newtheorem{definition}{Definition}

\newtheorem{lemma}{Lemma}

\title{}

\begin{document}

\twocolumn[
\icmltitle{Alleviating Privacy Attacks via Causal Learning}
\begin{icmlauthorlist}
  \icmlauthor{Shruti Tople}{msr}
  \icmlauthor{Amit Sharma}{msr}
	\icmlauthor{Aditya V. Nori}{msr}
\end{icmlauthorlist}
\icmlaffiliation{msr}{Microsoft Research}

\icmlcorrespondingauthor{Shruti Tople}{shruti.tople@microsoft.com}
\icmlcorrespondingauthor{Amit Sharma}{amshar@microsoft.com}
\icmlkeywords{Machine Learning, ICML}
\vskip 0.3in

]

\printAffiliationsAndNotice{}

\begin{abstract}
Machine learning models, especially deep neural networks have been shown to be  susceptible to privacy attacks such as \textit{membership inference} where an adversary can detect whether a data point was used for training a black-box model. Such privacy risks are exacerbated when a model's predictions are used on an unseen data distribution. 
To alleviate privacy attacks, we demonstrate  the benefit of predictive models  that are based on the causal relationships between input features and the outcome. 
We first show that models learnt using causal structure generalize better to unseen data, especially on data from different distributions than the train distribution. Based on this generalization property, we establish a theoretical link between causality and privacy: compared to associational models, causal models provide stronger differential privacy guarantees and are more robust to membership inference attacks. Experiments on simulated Bayesian networks and the colored-MNIST dataset show that 
 associational models exhibit upto 80\% attack accuracy under different test distributions and sample sizes whereas causal models exhibit attack accuracy close to a random guess.

\end{abstract}


\section{Introduction}
Machine learning algorithms, especially deep neural networks (DNNs) have found diverse applications in various fields such as healthcare~\citep{esteva2019guide}, gaming~\citep{mnih2013playing}, and finance~\citep{tsantekidis2017using,fischer2018deep}. However, a line of recent research has shown that 
deep learning algorithms are susceptible to privacy attacks that leak information about the training dataset~\citep{fredrikson2015model, rahman2018membership,song2018the,hayes2017logan}. Particularly, one such attack called \emph{membership inference} reveals whether a  data sample was present in the training dataset~\citep{shokri2017membership}. The privacy risks due to membership inference elevate when the DNNs are trained on sensitive data such as in healthcare applications. For example, HIV patients would not want to reveal their participation in the training dataset. 

Membership inference attacks are shown to exploit overfitting of the model on the training dataset~\citep{yeom2018privacy}. Existing defenses propose the use of generalization techniques such as adding learning rate decay, dropout or using adversarial regularization techniques~\citep{nasr2018machine,salem2018ml}. 
All these approaches assume that the test and the training data belong to the same distribution. In practice, a model trained using data from one distribution is often used on  a (slightly) different distribution.  For example, hospitals in one region may train a model and share it with hospitals in different regions.  However, generalizing to a new context is a challenge for any machine learning model. We extend the scope of membership inference attacks to different distributions and show that the risk increases for associational models as the test distribution is changed. 

\textbf{Our Approach.} To alleviate privacy attacks, we propose using models that depend on the causal relationship between input features and the output.  Causal learning has been used to optimize for fairness and explainability properties of the predicted output~\citep{kusner2017counterfactual,nabi2018fair,datta2016algorithmic}. However, the connection of causal learning to enhancing privacy of models is yet unexplored. To the best of our knowledge, we provide the  first analysis of privacy benefits of causal models. 
By definition, causal relationships are invariant across input distributions~\citep{peters2016causal}, and therefore  predictions of \emph{causal models}  should be
independent of the observed data distribution, let alone the observed dataset. Thus, causal models  generalize better even with changes in the data distribution.

In this paper, we show that the generalizability property of causal models directly ensures better privacy guarantees for the input data. Concretely, we prove that with reasonable assumptions, {\bf a causal model always provides stronger (i.e., smaller $\epsilon$ value) differential privacy  guarantees than an associational model trained on the same features and with the same amount of added noise}. 
    Consequently, we show that {membership attacks  are ineffective (almost a random guess) on causal models trained on infinite samples}. 

Empirical attack accuracies on four different tabular datasets and the colored MNIST image dataset~\cite{arjovsky2019invariant} confirm our theoretical claims. On tabular data, we find that $60$K training samples are sufficient to reduce the attack accuracy of a causal model to a random guess. In contrast, membership attack accuracy for neural network-based associational models increases up to $80\%$ as the test distribution  is changed. 
On colored MNIST dataset, we find that attack accuracy for causal model is close to a random guess ($50\%$) compared to $66\%$ for an associational model under a shift in the data distribution.

To summarize, our main contributions include:
\begin{itemize}
\item For the same amount of added noise, models learned using causal structure provide stronger $\epsilon$-differential privacy guarantees than corresponding associational models.
\item Models trained using causal features are {\em provably} more robust to membership inference attacks than typical associational models such as neural networks. 
\item On the colored MNIST dataset and simulated Bayesian Network datasets where the test distribution may not be the same as the training distribution, the membership inference attack accuracy of causal models is close to  a ``random guess'' (i.e., $50\%$)  whereas associational models exhibit $ 65$-$80\%$ attack accuracy. 
\end{itemize}

\section{Generalization Property of Causal Models}
\label{sec:causal-properties}

Causal models  generalize well since their output depends on stable, causal relationships between  input features and the outcome instead of  associations between them~\citep{peters2016causal}. Our goal is to study the effect of this generalization property on  privacy of training data. 

\begin{figure}[t]
\centering
        \includegraphics[scale=0.20]{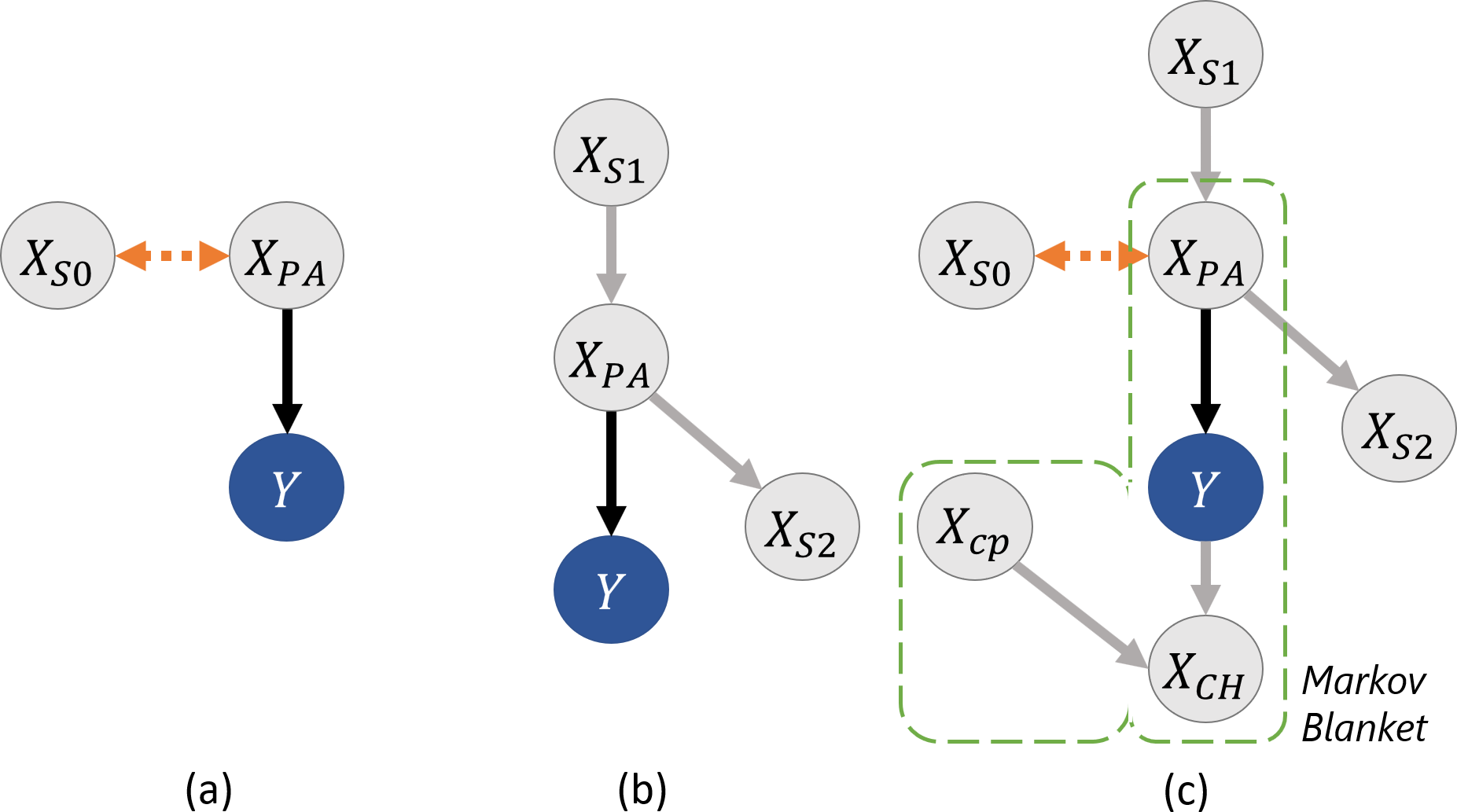}
        \vspace{-0.25cm}
    \caption{ Structural causal model where a  directed edge denotes a causal relationship; dashed bidirectional edges denote correlation. A causal predictive model includes only the parents of $\tt {Y: X_{PA}}$ [(a) and (b)]. Panel (c) shows the Markov Blanket of $\tt Y$.}
        \label{fig:causal-graphs}
\end{figure}

\subsection{Background: Causal Model}
Intuitively, a causal model identifies a subset of features that have a causal relationship with the outcome and learns a function from the subset to the outcome. To construct a causal model, one may use a structural causal graph based on domain knowledge that defines causal features as parents of the outcome under the graph.
Alternatively, one may exploit the strong relevance property from \cite{pellet2008markovb}, use score-based learning algorithms~\citep{scutari2009bnlearn} or recent methods for learning invariant relationships from training datasets from different distributions~\citep{peters2016causal,arjovsky2019invariant,bengio2019meta,mahajan2020}, or learn based on a combination of randomized experiments and observed data. 
Note that this is different from training probabilistic graphical models, wherein an edge conveys an associational relationship. Further details on causal models are in~\cite{pearl2009causality,peters2017elements}.
\begin{figure*}
    \begin{subfigure}{.33\textwidth}
        \centering
        \includegraphics[scale=0.15]{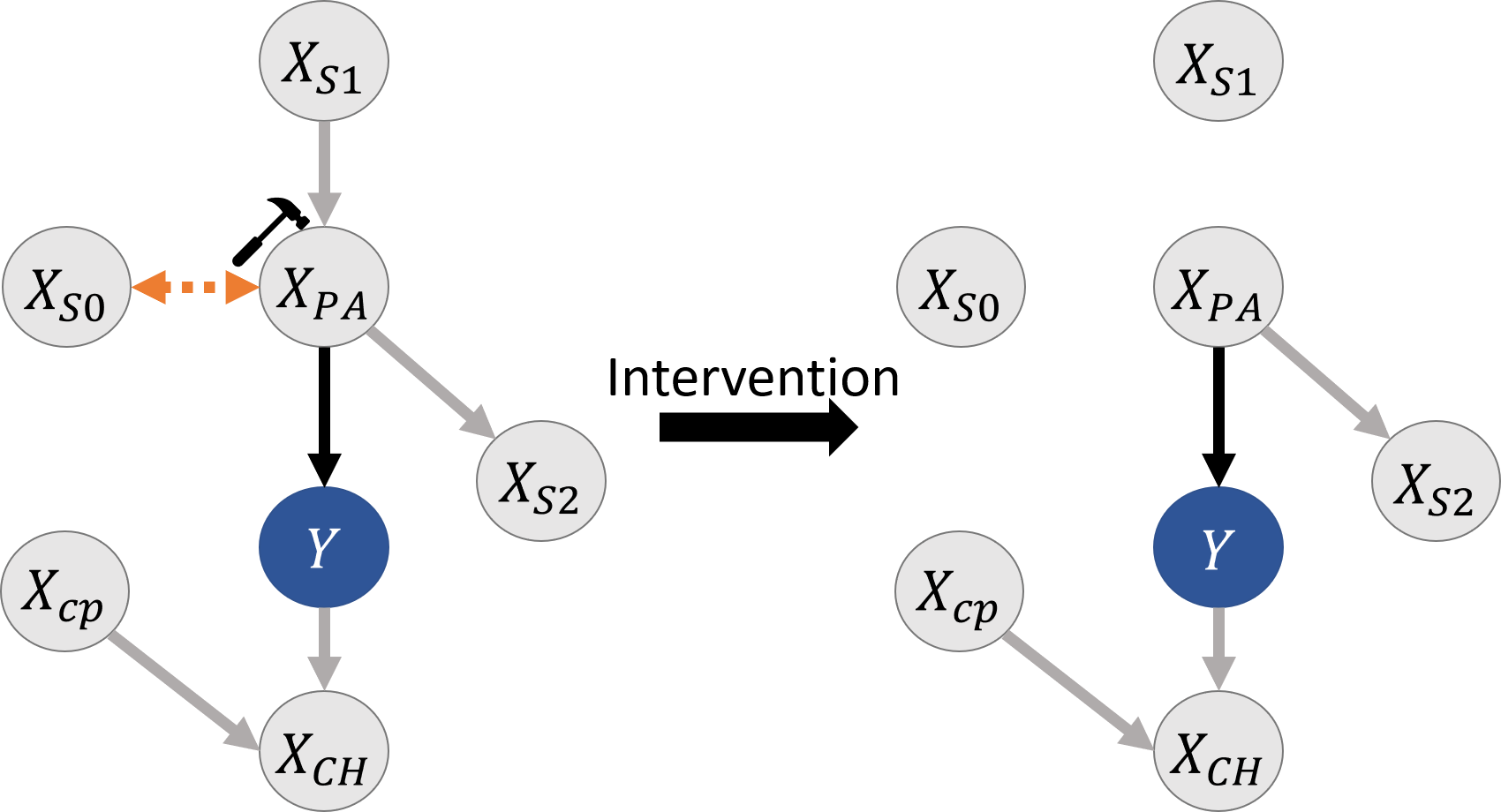}
	    \caption{ }
        \label{fig:intervention_pa}
    \end{subfigure}%
    \hfill
    \begin{subfigure}{.33\textwidth}
        \centering
        \includegraphics[scale=0.15]{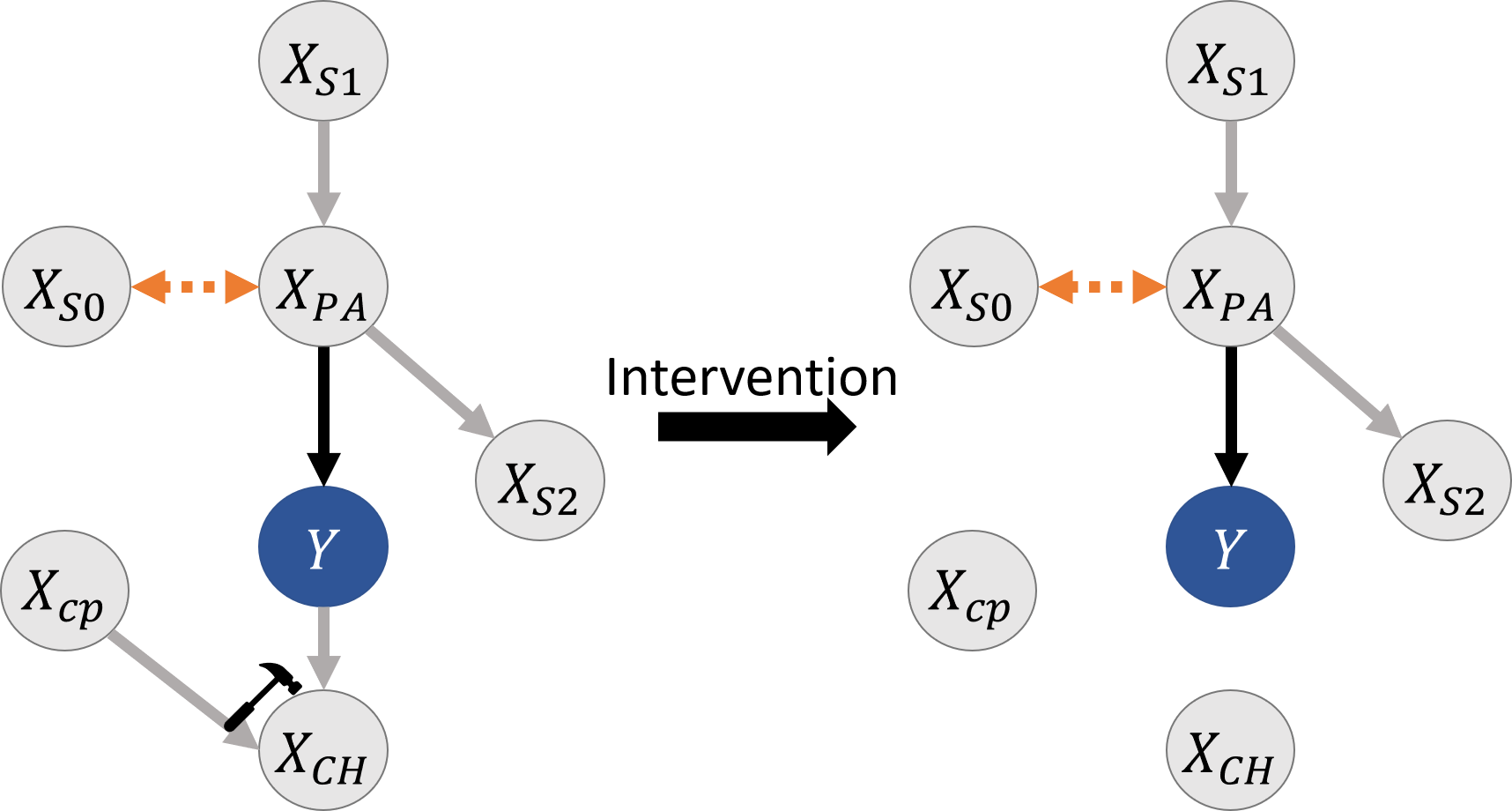}
	    \caption{ }
        \label{fig:intervention_ch}
    \end{subfigure}%
    \hfill
    \begin{subfigure}{.33\textwidth}
        \centering
        \includegraphics[scale=0.15]{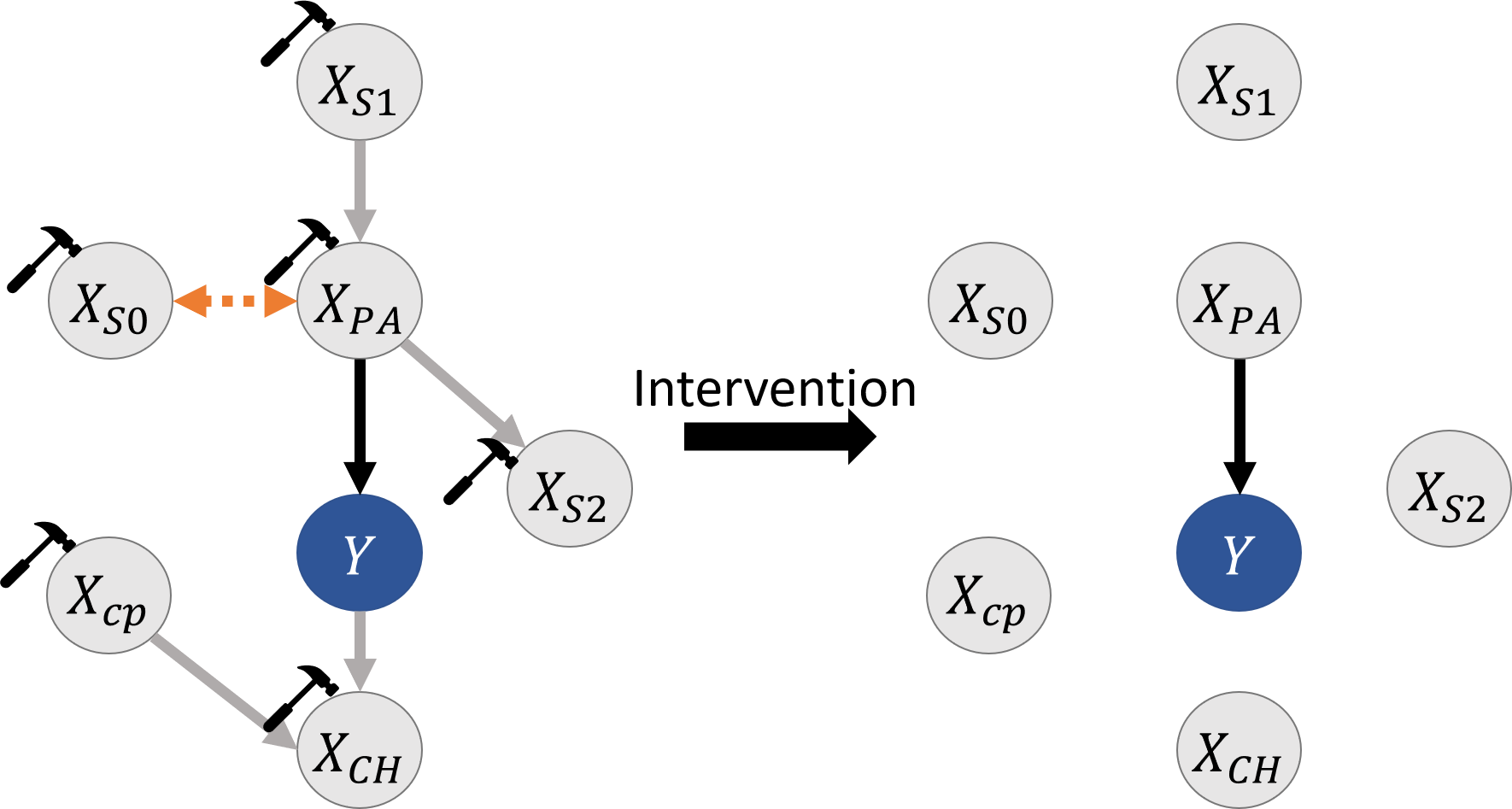}
        \caption{}
        \label{fig:intervention_all}
    \end{subfigure}
    \vspace{-0.25cm}
    \caption{Interventions on (a) parents of $\tt Y$, (b) children of $\tt Y$, and (c) all features. The black hammer denotes an intervention and each right subfigure shows the resultant causal model. Relationship between causal features and $\tt Y$, $\tt {Y=f(X_{PA})}$ remains invariant under all interventions but the relationship between  other features and $\tt Y$ varies based on the intervention. }
\label{fig:intervention}
\vspace{-0.4cm}
\end{figure*}

For ease of exposition, we assume the structural causal graph framework throughout. 
Consider data from a distribution $\tt (X,Y)\sim {P}$ where $\tt X$ is a $k$-dimensional vector. Our goal is to learn a function $\tt h(X)$ that predicts $\tt Y$.  Figure~\ref{fig:causal-graphs} shows causal graphs that denote the different relationships between $\tt X$ and $\tt Y$. Nodes of the graph represent variables and a directed edge  represents a direct causal relationship from a source to target node.   
Denote  $\tt X_{PA} \subseteq X$, the parents of $\tt Y$ in the causal graph. 
Fig.~(\ref{fig:causal-graphs}a) shows the scenario where $\tt X$ contains variables $\tt X_{S0}$ that are correlated to $\tt X_{PA}$ in $\tt {P}$, but not necessarily connected to either $\tt X_{PA}$ or $\tt Y$. These correlations may change in the future, therefore a generalizable model should not include these features. Similarly, Fig.~(\ref{fig:causal-graphs}b) shows parents and children of $\tt X_{PA}$. The d-separation principle states that a node is independent of its ancestors conditioned on  all its parents~\citep{pearl2009causality}. Thus, $\tt Y$ is independent of $\tt X_{S1}$ and $\tt X_{S2}$ conditional on $\tt X_{PA}$. Including them in a model does not add predictive value (and further, avoids prediction error when the relationships between $\tt X_{S1}$, $\tt X_{S2}$ and $\tt X_{PA}$ change).

 The key insight is that building a model for predicting $\tt Y$ using its parents $\tt X_{PA}$ ensures that the model generalizes to other distributions of $\tt X$, and also to changes in other causal relationships between $\tt X$, as long as the causal relationship of $\tt X_{PA}$ to $\tt Y$ is stable. We call such a model  a \emph{causal} model, the features in ($\tt X_{C} =X_{PA}$)  the \emph{causal features}, and assume that all causal features for $\tt Y$ are observed. In contrast, an \emph{associational} model uses all the available features.

Here we would like to distinguish causal features from $\tt Y$'s Markov Blanket. The Markov Blanket~\cite{pellet2008markovb} for $\tt Y$ contains its parents, children and parents of children.  Conditioned on its Markov blanket (Fig.~\ref{fig:causal-graphs}c),  $\tt Y$ is independent of all other variables in the causal graph, and therefore past work~\cite{aliferis2010local} suggests to build a predictive model using the features in $\tt Y$'s Markov Blanket\footnote{In some cases, it may be necessary to use Y's children for prediction, e.g., in predicting disease based on its symptoms. However, such a model will not generalize under intervention--- it makes an implicit assumption that symptoms will never be intervened upon, and that all causes of symptoms are observed.}. However, such a model is not robust to interventions. For instance, if there is an intervention on $\tt Y$'s children in a new domain (Fig.~\ref{fig:intervention}b), it will break the correlation between $\tt Y$ and $\tt X_{CH}$ and lead to incorrect predictions. To summarize, Fig.~(\ref{fig:intervention}c) demonstrates how a causal model based on parents is robust to all interventions on $\tt X$, unlike an associational model built using the Markov Blanket or other features.  

\subsection{Generalization to New Distributions}
We state the generalization property of causal models and show how it results in a stronger differential privacy guarantee. 
We first define \emph{In-distribution} and \emph{Out-of-distribution} generalization error.
Throughout, $\tt L(.,.)$ refers to the loss on a single input and $\tt \mathcal{L}_P(.,.)= \mathbb{E}_P L(.,.)$ refers to the expected value of the loss over a distribution $\tt P(X,Y)$. We refer to 
$\tt h:X \rightarrow Y$ as the hypothesis function or simply the model.  Then, $\tt L(h, h')$ is a loss function quantifying the difference between any two models  $\tt h$ and $\tt h'$. 

\begin{restatable}{definition}{ide} \label{def:ide} 
{\bf In-Distribution Generalization Error ($\tt IDE$)}.   
    Consider a dataset $\tt {S} \sim \tt P(X, Y)$. Then for a model $\tt h:X \rightarrow Y$ trained on $\tt {S}$, the in-distribution generalization error is given by:
    \begin{equation}
       \tt  IDE_{P} (h, y) = \mathcal{L}_{P}(h, y) - \mathcal{L}_{{S} \sim P}(h, y)
    \end{equation}
\end{restatable}

\begin{restatable}{definition}{ode}\label{def:ode}
 {\bf Out-of-Distribution Generalization Error ($\tt ODE$).}
    Consider a dataset $\tt {S}$ sampled from a distribution $\tt P(X, Y)$. Then for a model $\tt h:X \rightarrow Y$ trained on $\tt {S}$, the out-of-distribution generalization error with respect to another distribution  $\tt P^*(X, Y)$ is given by:
        \vspace{-0.2cm}
    \begin{equation}
        \tt ODE_{P,P^*} (h,y) = \mathcal{L}_{P^*}(h, y) - \mathcal{L}_{{S} \sim P}(h, y)
    \end{equation} 
\end{restatable}

\begin{restatable}{definition}{disc}\label{def:discrepancy}
{\bf Discrepancy Distance ($\tt disc_L$) (Def. 4 in \cite{mansour2009domain})}.
    Let $\tt \mathcal H$ be a set of  models,  $\tt h: X \rightarrow Y$. Let $\tt L:Y \times Y \rightarrow \mathbb{R_+}$ define a loss function over $\tt Y$ for any such model $\tt h$. Then the discrepancy distance $\tt disc_L$  over any two distributions $\tt P(X,Y)$ and $\tt P^*(X,Y)$ is given by:
    \begin{equation}
        \tt disc_{L,\mathcal{H}}(P, P^*)  = \max_{h,h' \in \mathcal H} |\mathcal{L}_P(h, h') - \mathcal{L}_{P^*}(h, h')|
    \end{equation}
\end{restatable}
Intuitively,  the term $\tt disc_L(P, P^*)$ denotes the distance between the two distributions. Higher the distance, higher is the chance of an error when transferring  model $\tt h$ from one distribution to another. Next we state the theorem on the generalization property of causal models.

\begin{restatable}{theorem}{generalitytheorem} \label{thm:causal-gen-bounds}
    Consider a structural causal graph $\tt G$ that connects $\tt X$ to $\tt Y$, and  causal features $\tt X_C \subset X$ where $\tt X_C$ represent the parents of $\tt Y$ under $\tt G$. Let $\tt P(X, Y)$ and $\tt P^*(X, Y)$ be two distributions with arbitrary $\tt P(X)$ and $\tt P^*(X)$, having overlap, ${\tt P(X=x)>0}$  whenever $\tt P^*(X=x)>0$.  
    In addition,  the causal relationship between $\tt X_C$ and $\tt Y$ is preserved, which implies that $\tt P(Y|X_C)=P^*(Y|X_C)$.  
    Let $L$ be a symmetric loss function that obeys the triangle inequality (such as L1, L2 or 0-1 loss), and let $\tt f:X_C\rightarrow Y$ be the optimal predictor among all hypotheses using $\tt X_C$ features under $L$, i.e., $\tt f = \arg \min_h L_{x_c}(y,h(x_c))$ for all $\tt x_c$,  and thus $f$ depends only on $\tt \Pr(Y|X_C)$ (e.g., $\tt f:=\mathbb{E}[Y|X_C]$ for L2 loss).
 Further, assume that $\tt \mathcal H_C$ represents the set of {causal} models ${\tt h_c:X_C \rightarrow Y}$ that may use all causal features and $\tt \mathcal H_A$ represent the set of {associational} models ${\tt h_a: X \rightarrow Y}$ that may use all available features, such that ${\tt f \in \mathcal{H}_C}$ and ${\tt \mathcal{H}_C \subseteq \mathcal{H}_A}$.  
\begin{enumerate}
    \item When generation of $\tt Y$ is deterministic, $\tt y=f(X_c)$ (e.g., when $\tt Y|X_C$ is almost surely constant), the $\tt ODE$ loss for a causal model  $\tt h_c \in \mathcal H_C$ is bounded by:
\begin{align}
        \tt ODE_{P,P^*}(h_c, y)   \tt = \mathcal{L}_{P^*}(h_c, y) -  \mathcal{L}_{{S}\sim P}(h_c, y) \nonumber \\
                                  \tt \leq disc_{L, \mathcal{H}_C}(P, P^*) + IDE_P(h_c, y) 
\end{align}
       Further, for any $\tt P$ and $\tt P^*$, the upper bound of $\tt ODE$ from a dataset $\tt{S}\sim \tt P(X,Y)$ to $\tt P^*$(called $\tt ODE\texttt{-}Bound$)  for a causal model $\tt h_c \in \mathcal H_C$ is less than or equal to the upper bound  $\tt ODE\texttt{-}Bound$ of an associational model $h_a \in \mathcal H_A$,   with probability at least $(1-\delta)^2$.
 \begin{equation*}
      \tt  ODE\texttt{-}Bound_{P,P^*}(h_c, y; \delta) \leq ODE\texttt{-}Bound_{P,P^*}(h_a, y; \delta)
 \end{equation*}
    \item When generation of $\tt Y$ is probabilistic, the $\tt ODE$ error for a causal model  $\tt h_c \in \mathcal H_C$ includes additional terms for the loss between  $\tt Y$ and optimal causal models $\tt h_{c,P}^{OPT}=h_{c,P^*}^{OPT}$ on $\tt P$ and $\tt P^*$ respectively.
\begin{align}
    \tt ODE_{P,P^*}(h_c, y )  
                                       & \tt \leq disc_{L, \mathcal{H}_C}(P, P^*) + IDE_P(h_c, y) + \nonumber \\
                                       & \tt \mathcal{L}_{P^*}(h_{c,P^*}^{OPT}, y) +  \mathcal{L}_{P}(h_{c,P}^{OPT},y) 
\end{align}
However, while the loss of an associational model can be lower on $\tt P$, there always exists a $\tt P^*$ such that the worst case $\tt ODE\texttt{-}Bound$ for an associational model is higher than the same for a causal model. 
\footnotesize
\begin{equation*}
    \small \tt  \max_{P^*} ODE\texttt{-}Bound_{P,P^*}(h_c, y; \delta) \leq  \max_{P^*} ODE\texttt{-}Bound_{P,P^*}(h_a, y; \delta)
\end{equation*}
\normalsize
\end{enumerate}
\end{restatable}

{\em Proof Sketch.}
    As an example, consider a colored MNIST data distribution $\tt P$ such that the label $\tt  Y$ is assigned based on the shape of a digit. Here the shape features represent the causal features ($\tt X_C$). If the shape is closest to shapes for $\{0,1,2,3,4\}$ then $\tt Y=0$, else $\tt Y=1$. Additionally, all images classified as $1$ are colored with the same color (say red). Then, under a suitably expressive class of models,  the loss-minimizing associational model may use only the color feature to obtain zero error, while the loss-minimizing causal model still uses the shape (causal) features.  On any new $\tt  P^*$ that does not follow the same correlation of digits  with color, we expect that the  associational model will have higher error than the  causal model.

    Formally, since $\tt P(Y|X_C)=P^*(Y|X_C)$ and $\tt f \in \mathcal{H}_C$, the optimal causal model that minimizes loss over $\tt P$ is the same as the loss-minimizing model over $\tt P^*$. That is, $\tt h_{c,P}^{OPT} =h_{c,P^*}^{OPT}$. However for associational models, the optimal models may not be the same  $\tt h_{a,P}^{OPT} \neq h_{a,P^*}^{OPT}$ and thus there is an additional loss term when generalizing to data from $\tt P^*$. The rest of the proof follows from triangle inequalities on the loss function and the standard bounds for $\tt IDE$ ( in \app~\ref{app:gen_thm}).

    For individual instances, we present a similar result on the worst-case generalization error (proof in  \app~\ref{app:gen_cor}).
\begin{restatable}{theorem}{generalitycorollary} 
\label{cor:single-input}
Consider a causal model $\tt h_{c,S}^{min}: X_C \rightarrow Y$ and an associational model $\tt h_{a,S}^{min}: X \rightarrow Y$ trained on a dataset $\tt {S} \sim P(X, Y)$ with loss $L$. Let $\tt (x,y) \in S$ and $\tt (x',y') \notin S$ be two input instances such that they share the same true labelling function on the causal features, $\tt y \sim P(Y|X_C=x)$ and $\tt y' \sim P(Y|X_C=x')$.  
   Then, the worst-case generalization error for a causal model on such $\tt x'$  is less than or equal to that for an associational model. 
   \footnotesize
    \begin{equation*}
        \tt   \max_{x \in S, x'}  L_{x'}(h_{c,S}^{min}, y)  - L_{x}(h_{c,S}^{min}, y) \leq  \max_{x \in S, x'} L_{x'}(h_{a,S}^{min}, y)  - L_{x}(h_{a,S}^{min}, y)  
    \end{equation*}
    \normalsize
\end{restatable}

\section{Main Result: Privacy  with Causality}
\label{privacy-attacks}
We now present our main result on the privacy guarantees and attack robustness of causal models. 

\subsection{Differential Privacy Guarantees}
Differential privacy~\citep{dwork2014algorithmic} provides one of the strongest  notions of privacy  to hide the participation of an individual sample in the dataset. To state informally, it ensures that the presence or absence of a single data point in the input dataset does 
not change the output by much. 

\begin{definition}[Differential Privacy]
    A mechanism $\mathcal{\tt M}$ with domain $\mathcal{I}$ and range $\mathcal{O}$ satisfies $\epsilon$-differential privacy if for any two datasets $d,d' \in \mathcal{I}$ that differ only in one input and for a set $\mathcal{S}\subseteq \mathcal{O}$, the following holds:
   $ \Pr(\mathcal{M}(d)\in \mathcal{S}) \leq e^\epsilon  \Pr(\mathcal{M}(d')\in \mathcal{S}) $
\end{definition}

The standard approach to designing a differentially private mechanism is by calculating the \emph{sensitivity} of an algorithm and adding noise proportional to the sensitivity. Sensitivity captures the change in the output of a function due to changing a single data point in the input. Higher the sensitivity, larger is the amount of noise required to make any function differentially private with reasonable $\epsilon$ guarantees. Below we provide a formal definition of sensitivity, derive a corollary based on  the generalization property from Theorem~\ref{cor:single-input}, and then show that sensitivity of a causal learning function is lower than or equal to an associational learning function (proofs are in \app~\ref{app:sensitivity}).

\begin{definition}
    [Sensitivity (From Def. 3.1 in \cite{dwork2014algorithmic}]
    Let $\tt \mathcal F$ be a function that maps a dataset to a vector in $\mathbb{R}^d$. Let $\tt {S}$, $\tt {S'}$ be two datasets such that  $\tt {S'}$ differs from $\tt {S}$ in one data point.
Then the $l_1$-sensitivity of a function $\tt \mathcal F $ is defined as: 
    $\tt    \Delta \mathcal F = \max_{{S}, {S'}}||\mathcal F (S)  - \mathcal F (S')||_1$
\end{definition}

\begin{restatable}{corollary}{generalizeneighboringdatasets}
    \label{cor:gen-neighbor-datasets}
    Let $\tt S$ be a dataset of $\tt n$ $\tt (x,y)$ values, such that $\tt y^{(i)} \sim P(Y|X_C=x^{(i)})  \forall (x^{(i)}, y^{(i)}) \in S$, where $\tt P(Y|X_C)$ is the invariant conditional distribution on the causal features $\tt X_C$. Consider a neighboring dataset $\tt S'$ such that $\tt S' = S \backslash (x,y) + (x',y')$ where $\tt (x,y) \in S$,  $\tt (x',y') \notin S$, and $\tt (x',y')$ shares the same conditional distribution $\tt y' \sim P(Y|X_C=x_c')$. Then the maximum generalization error from $\tt S$ to $\tt S'$ for a causal model trained on $\tt S$ is lower than or equal to that of an associational model. 
    \footnotesize
    \begin{equation*}
        \tt   \max_{S, S'}  \mathcal L_{S'}(h_{c,S}^{min}, y)  - \mathcal L_{S}(h_{c,S}^{min}, y) \leq  \max_{S,S'} \mathcal L_{S'}(h_{a,S}^{min}, y)  - \mathcal L_{S}(h_{a,S}^{min}, y)  
    \end{equation*}
    \normalsize
\end{restatable}
\begin{restatable}{lemma}{sensitivitylemma} 
\label{lem:theta-sensitivity}
Let $\tt S$ and $\tt S'$ be two datasets defined as in Corollary~\ref{cor:gen-neighbor-datasets}. 
    Let a model $\tt h$ be specified by a set of parameters $\theta \in \Omega \subseteq \mathbb{R}^n$. Let $\tt h_S^{min}(x;\theta_S)$ be a model learnt using $\tt S$ as training data and $\tt h_{S'}^{min}(x; \theta_{S'})$ be the model learnt using $\tt S'$ as training data, using a loss function $\tt L$ that is $\lambda$-strongly convex over $\Omega$, $\rho$-Lipschitz, symmetric and obeys the triangle inequality. 
    Then, under the conditions of Theorem~\ref{thm:causal-gen-bounds} (optimal predictor $f\in \mathcal{H}_C$) and for a sufficiently large $\tt n$, the sensitivity of a causal learning function $\mathcal F_c$ that outputs learnt empirical model $\tt h_{c,S}^{min} \leftarrow \mathcal F_c(S)$ and $\tt h_{c,S'}^{min} \leftarrow \mathcal F_c(S')$ is lower than or equal to the sensitivity of an associational learning function $\mathcal F_a$ that outputs
      $\tt h_{a,S}^{min} \leftarrow \mathcal F_a(S)$ and $\tt h_{a,S'}^{min} \leftarrow \mathcal F_a(S')$,
    \begin{equation*}
      \tt  \Delta \mathcal F_c = \max_{{S}, {S'}}||  h_{c, S}^{min} -  h_{c, S'}^{min}||_1 
            \leq \max_{{S},{S'}}|| h_{a, S}^{min} - h_{a, S'}^{min}||_1 = \Delta \mathcal F_a 
    \end{equation*}
    where the maximum is over all such datasets $\tt S$ and $\tt S'$.
\end{restatable}

We now prove our main result on differential privacy.

\begin{theorem}\label{thm:main}
    Let $\ \hat{\mathcal F_c}$ and $\hat{\mathcal F_a}$ be the differentially private mechanisms, \newtext{obtained by adding Laplace noise to model parameters} of the causal learning and associational learning functions $ {\mathcal F_c}$ and $ {\mathcal F_a}$ respectively. Let $\hat{\mathcal F_c}$ and $\hat{\mathcal F_a}$ provide $\tt \epsilon_c$-DP and  $\tt \epsilon_a$-DP guarantees respectively. Then, for equivalent noise added to both the functions and sampled from the same distribution, $\tt Lap(Z)$, we have $\tt \epsilon_c \leq \tt \epsilon_a	$.

\end{theorem}

\begin{proof}
According to the Def. 3.3 of Laplace mechanism from \cite{dwork2014algorithmic}, we have,
\begin{equation*}
 \hat{\mathcal F_c} {\tt = \mathcal F_c + \mathcal K \sim Lap (\frac{\Delta \mathcal F_c}{\epsilon_c}) } \qquad 
 \hat{\mathcal F_a} {\tt = \mathcal F_a + \mathcal K \sim Lap (\frac{\Delta \mathcal F_a}{\epsilon_a})} 
\end{equation*}
The noise is added to the output of the learning algorithm $\tt \mathcal F(.)$ i.e., the model parameters.
Since  $\mathcal K$ is sampled from the same noise distribution,
\begin{equation}
\tt  Lap (\frac{\Delta \mathcal F_c}{\epsilon_c}) =  Lap (\frac{\Delta \mathcal F_a}{\epsilon_a}) \qquad  \therefore \frac{\Delta \mathcal F_c}{\epsilon_c} = \frac{\Delta \mathcal F_a}{\epsilon_a}
\end{equation}
From Lemma~\ref{lem:theta-sensitivity}, $\Delta \mathcal F_c \leq \Delta \mathcal F_a$ and hence $\tt \epsilon_c \leq \tt \epsilon_a$.
\end{proof}

While we prove the general result above, our central claim comparing differential privacy for causal and associational models also holds for mechanisms that provide a tighter data-dependent differential privacy guarantee~\citep{papernot2016teacher-dp}. The key idea is to produce an output label based on voting from $\tt M$ teacher models, each trained on a disjoint subset of the training data.  We state the theorem below  and provide its proof in \app~\ref{app:voting}. Given  datasets from different domains, the below theorem also provides a \emph{constructive} proof to train a differentially private causal algorithm following the method from \citet{papernot2016teacher-dp}.

\begin{restatable}{theorem}{votingtheorem}\label{thm:voting}
    Let $\tt {D}$ be a dataset generated from  possibly a mixture of different distributions $\tt \Pr(X,Y)$ such that $\tt \Pr(Y|X_C)$ remains the same. Let $\tt n_j$ be the votes for the jth class from $\tt M$ teacher models. Let $\tt \mathcal{M}$ be the mechanism that produces a noisy max,  $\tt \arg \max_j \{ n_j + Lap(2/\gamma)\}$. Then the privacy budget $\tt \epsilon_c$ for a causal model trained on $\tt D$ is lower than that for an associational model with the same accuracy. 
\end{restatable}

\subsection{Robustness to Membership Attacks}\label{sec:robust-membership}
Deep learning models have been shown to memorize or overfit on the training data during the learning process~\citep{carlini2018secret}. Such overfitted models are susceptible to {\em membership inference attacks} that can accurately predict whether a target input belongs to the training dataset or not~\citep{shokri2017membership}. 
There are multiple variants of the attack depending on the information accessible to the adversary. An adversary with black-box access to a  model observes confidence scores for the predicted output whereas one with the white-box access observes  all model parameters and the output at each layer in the model~\citep{nasr2018comprehensive}. In the black-box setting, a membership attack is possible whenever the distribution of output scores for training data is different from the test data, and has been connected to model overfitting~\citep{yeom2018privacy}. Alternatively, if the adversary knows the distribution of the training inputs, they may learn a ``shadow'' model based on synthetic inputs and use the shadow model's output to build a membership classifier~\citep{shokri2017membership}. For the white-box setting, if an adversary knows the true  label for the target input, then they may guess membership of the input based on either the loss or gradient values during inference~\citep{nasr2018comprehensive}. 

Most of the existing membership inference attacks have been demonstrated for test inputs from the same data distribution as the training set. When test inputs are expected from the same distribution, methods to reduce overfitting (such as adversarial regularization) can help reduce privacy risks~\citep{nasr2018machine}. However in practice, this is seldom the case. For instance, in our example of a model trained with a single hospital's data, the test inputs may come from different hospitals. 
Therefore, models trained to reduce the generalization error for a specific test distribution are still susceptible to membership inference when the distribution of features is changed. This is due to the problem of \emph{covariate shift}  that introduces a domain adaptation error term~\citep{mansour2009domain}. That is, the loss-minimizing model that predicts $\tt Y$ changes with a different distribution, and thus allows the adversary to detect differences in losses for the test versus training datasets. As we show below, causal models alleviate the risk of membership inference attacks.   
Based on~\citet{yeom2018privacy}, we first define a membership attack.
 \begin{definition}
     Let  $\tt h$ be trained on a dataset $\tt {S}(X,Y) \sim P$ of size $\tt N$. Let $\mathcal{A}$ be an adversary with access to  $\tt h$ and an input $\tt  x\sim P^*$ where $\tt  P^*$ is any distribution such that $\tt P(Y|X_C)=P^*(Y|X_C)$. Then advantage of an adversary in membership inference is the difference between true and false positive rate in guessing whether the the input belongs to the training set. 
$\tt Adv(\mathcal{A}, h, N, P, P^*) = \Pr[\mathcal{A}=1|b=1] - \Pr[\mathcal{A}=1|b=0]$, 
where $b=1$ if the input is in the training set and else is $0$. 
\end{definition}

As a warmup, we demonstrate the relationship between membership advantage and out-of-distribution generalization using a specific  adversary that predicts membership for an input based on the model's loss. 
This adversary is motivated by empirical membership inference algorithms~\citep{shokri2017membership,nasr2018comprehensive}. 
\begin{restatable}{definition}{boundedlossadversary}\label{def:boundedloss-adversary}[From \cite{yeom2018privacy}]
    Assume that the loss $L$ is bounded by $\tt B\in\mathbb{R}^+$. Then for a model $\tt h$ and an input $\tt x$,  a Bounded-Loss adversary $\tt \mathcal{A}_{BL}$ predicts membership in a train set  with probability $\tt 1-L_x(h,y)/B$.
\end{restatable}

\begin{restatable}{theorem}{boundedlossadvantage}\label{thm:boundedloss-advantage}
    Assume a training set $\tt S$ of size $\tt n$ and a loss function $\tt L$ that is bounded by $\tt B\in\mathbb{R}^+$. Under the conditions of Theorem~\ref{thm:causal-gen-bounds} and for a Bounded-Loss adversary $\tt \mathcal A_{BL}$, the worst-case membership advantage of a causal model $\tt h_{c,S}^{min}$ is lower than that of an associational model $\tt h_{a,S}^{min}$. 
    \begin{equation*}
        \tt \max_{P^*} Adv(\mathcal{A}_{BL}, h_{c,S}^{min}, n, P, P^*) \leq \max_{P^*} Adv(\mathcal{A}_{BL}, h_{a,S}^{min}, n, P, P^*)
    \end{equation*}
\end{restatable}
\begin{proof}
Let the variable  $\tt b=1$ denote that data point belongs to the train dataset $\tt S$.   The membership advantage of the bounded loss adversary $\tt \mathcal A$ for any model $\tt h$ trained on dataset $\tt S\sim P$ is given by,
    \begin{equation*} \label{eq:adv-upper-bound}
        \begin{split}
            \tt  Adv(&\tt\mathcal{A}, h, n, P, P^*)  =  \Pr[\mathcal{A}=1|b=1] - \Pr[\mathcal{A}=1|b=0] \\
            & \tt =  \Pr[\mathcal{A}=0|b=0] - \Pr[\mathcal{A}=0|b=1] \\
            & \tt = \mathbb{E}[\frac{L_{x'}(h,y)}{B}|b=0] - \mathbb{E}[\frac{L_x(h, y)}{B}|b=1]\\
            & \tt = \frac{1}{B}(\mathbb{E}_{x'\sim P^*}[L_{x'}(h,y)] - \mathbb{E}_{x\sim S}[L_{x}(h,y)])  \\
            &\tt \leq \max_{x' \not \in S}L_{x'}(h,y) - \mathcal{L}_{S}(h,y)
        \end{split}
    \end{equation*}
    where  the third equality is due to Def.~\ref{def:boundedloss-adversary} for $\mathcal A_{\tt BL}$,  
    and the last inequality is due to the fact that the expected value of a random variable is less than or equal to the maximum value.
    Note that the upper bound in the above inequality is tight: it can be achieved by evaluating membership advantage only on those $\tt x'$ that lead to the maximum loss difference. Thus, 
    \begin{equation}\label{eqn:max-advantage-loss}
        \tt \max_{P^*} Adv(\mathcal{A},h, N, P, P^*) =  \max_{x'}L_{x'}(h,y) - \mathcal{L}_{S}(h,y)
    \end{equation}
    
    Applying Eqn.~\ref{eqn:max-advantage-loss} to the trained causal model $\tt h_{c,S}^{min}$ and associational model  $\tt h_{a,S}^{min}$, we obtain:
    \footnotesize 
    \begin{equation*} \label{eq:max-adv-causal-assoc}
        \begin{split}
            \tt \max_{P^*} Adv(\mathcal{A},h_{c,S}^{min}, n, P, P^*) =  \max_{x'}L_{x'}(h_{c,S}^{min},y) - \mathcal{L}_{S}(h_{c,S}^{min},y)\\
            \tt \max_{P^*} Adv(\mathcal{A},h_{a,S}^{min}, n, P, P^*) =  \max_{x'}L_{x'}(h_{a,S}^{min},y) - \mathcal{L}_{S}(h_{a,S}^{min},y)
    \end{split}
    \end{equation*}
\normalsize
    From Theorem~\ref{cor:single-input} proof (Suppl. Eqn.~\ref{eq:single-input-max-avg}), we state the inequality,
$\tt \max_{x'}L_{x'}(h_{c,S}^{min},y) - \mathcal{L}_{S}(h_{c,S}^{min},y) \leq  \max_{x'}L_{x'}(h_{a,S}^{min},y) - \mathcal{L}_{S}(h_{a,S}^{min},y)$. 
    Combining this inequality with the above equations, we get the main result. 
    \begin{equation*}
        \tt \max_{P^*} Adv(\mathcal{A},h_{c,S}^{min}, n, P, P^*) \leq \max_{P^*} Adv(\mathcal{A},h_{a,S}^{min}, n, P, P^*)
    \end{equation*}
\end{proof}

We now prove a more general result. The {maximum} membership advantage for a causal DP mechanism (based on a causal model) is not greater than that of an associational DP mechanism. We present a lemma from \citet{yeom2018privacy}.

\begin{lemma}\label{lem:yeom} [From~\cite{yeom2018privacy}]
    Let $\mathcal{M}$ be a $\epsilon$-differentially private mechanism based on a model $h$. The membership advantage is bounded by $\exp(\epsilon) - 1$.
\end{lemma}
Based on the above lemma and Theorem~\ref{thm:main}, we can show that the upper bound of membership advantage for an $\epsilon_c$-DP mechanism from a causal model $e^{\epsilon_c}-1$ is not greater than that of an $\epsilon_a$-DP mechanism from an associational model, $e^{\epsilon_a}-1$, since $\epsilon_c \leq  \epsilon_a$. The next theorem proves that the same holds true  for the {\em maximum} membership advantage.

\begin{restatable}{theorem}{maxadvdp} \label{thm:max-madv-dp}
    Under the conditions of Theorem~\ref{thm:causal-gen-bounds}, 
    let $\tt S \sim P(X,Y)$ be a dataset sampled from $\tt P$. Let $\tt \hat{\mathcal F}_{c,S}$ and $\tt \hat{\mathcal F}_{a,S}$ be the differentially private mechanisms trained on $\tt S$ by adding identical  Laplacian noise to the causal and associational learning functions from Lemma~\ref{lem:theta-sensitivity} respectively. Assume that a membership inference adversary is provided inputs sampled from either $\tt P$ or $\tt P^*$, where $\tt P^*$ is any distribution such that $\tt P(Y|X_C)=P^*(Y|X_C)$.   Then, across all adversaries $\mathcal{A}$ that predict membership in $\tt S\sim P$, the worst-case membership advantage of $\tt \hat{\mathcal{F}}_{\tt c,S}$ is  not greater than that of $\tt \hat{\mathcal{F}}_{\tt a,S}$. 
    \begin{equation*}
        \tt \max_{\mathcal{A},P^*} Adv(\mathcal{A}, \hat{\mathcal{F}}_{c,S}, n, P, P^*) \leq \max_{\mathcal{A},P^*} Adv(\mathcal{A}, \hat{\mathcal{F}}_{a,S}, n, P, P^*) 
    \end{equation*}
\end{restatable}
\begin{proof}
    We construct an expression for the maximum membership advantage for any $\epsilon$-DP model and then show that it is an increasing function of the sensitivity, and thus $\epsilon$.
\end{proof}

Finally, we show that membership advantage against a causal model trained on infinite data will be zero for any adversary.   The proof is based on the result from Theorem~\ref{thm:causal-gen-bounds} that $\tt h_{c,P}^{OPT} = h_{c, P^*}^{OPT}$ for a causal model. 
Crucially, membership advantage does not go to zero as $\tt n\to\infty$ for associational models, since   $ \tt h_{a,P}^{OPT} \neq h_{a, P^*}^{OPT}$ in general. Detailed proof is in \app ~\ref{appsec:infty-attack}.
\begin{restatable}{corollary}{inftyattack}\label{cor:infty-attack}
    Under the conditions of Theorem~\ref{thm:causal-gen-bounds}, let $\tt h_{c,S}^{min}$ be a causal model trained using empirical risk minimization on a dataset $\tt S \sim P(X, Y)$ with sample size $\tt n$. As $ \tt n \rightarrow \infty$,     membership advantage $\tt Adv(\mathcal{A}, h_{c,S}^{min}) \rightarrow 0$.
\end{restatable}

\subsection{Robustness to Attribute Inference Attacks}\label{model-inversion}
We prove similar results on the benefits of causal models for attribute inference attacks where
a model may reveal the value of sensitive features of a test input, given partial knowledge of its features. For instance, given a model's output and certain features about a person,   an adversary may infer other attributes of the person (e.g., their demographics or genetic information). As another example, it can be possible to infer a person's face based on hill climb on the output score of a face detection model~\citep{fredrikson2015model}. Model inversion is not always due to a fault in learning: a model may learn a true, generalizable relationship between features and the outcome, but still be vulnerable to a model inversion attack. This is because given $k-1$ features and the true outcome label, it is possible to guess the $k$th feature by brute-force search on output scores generated by the model. 

However, inversion based on learning correlations between features and the outcome, e.g., using demographics to predict disease, can be alleviated by causal models, since a non-causal feature will not be included in the model. 

\begin{definition} [From~\cite{yeom2018privacy}]
    Let $\tt h$ be a model trained on a dataset $\tt {S}$(X,Y). Let $\tt \mathcal A$ be an adversary with access to h, and a partial test input $\tt x_A \subset x$. The attribute advantage of the adversary is the difference between true and false positive rates in guessing the value of a sensitive feature $\tt x_s \notin x_A$. For a binary $\tt x_s$,
    \begin{equation*}
      \tt  Adv(\mathcal{A}, h) = Pr(\mathcal{A}=1|x_s=1) - Pr(\mathcal{A}=1|x_s=0)
    \end{equation*}
\end{definition}

\begin{restatable}{theorem}{attribute}
    Given a dataset $\tt {S}(X,Y)$ of size $n$ and a structural causal model that connects $\tt X$ to $\tt Y$, a causal model  $\tt h_c$ makes it impossible to infer non-causal features. 
\end{restatable}
The proof is in \app~\ref{model-inversion}.

\section{Implementation and Evaluation}
\label{eval}

\begin{table}[t]

    \centering
	\begin{adjustbox}{width=\columnwidth,center}
		\begin{tabular}{@{} c|c c c c @{} }
		\toprule
		\textbf{Dataset}                                                   & \textbf{Child} & \textbf{Alarm}                                            & \textbf{\begin{tabular}[c]{@{}c@{}} (Sachs)\end{tabular}} & \textbf{Water} \\ \midrule
		\textbf{\begin{tabular}[c]{@{}c@{}}Output\end{tabular}} & XrayReport     & \begin{tabular}[c]{@{}c@{}}BP\end{tabular} & Akt                                                                             & CKNI\_12\_45   \\ \midrule
		\textbf{No. of classes}                                            & 5              & 3                                                         & 3                                                                               & 3              \\ 
		\textbf{Nodes}                                                     & 20             & 37                                                        & 11                                                                              & 32             \\ 
		\textbf{Arcs}                                                      & 25             & 46                                                        & 17                                                                              & 66             \\ 
		\textbf{Parameters}                                                & 230            & 509                                                       & 178                                                                             & 10083          \\ \bottomrule
		\end{tabular}
		\end{adjustbox}
		\vspace{-0.2cm}
      \caption{{Details of the benchmark datasets.}}
\label{tbl:dataset}
\end{table}

We perform our evaluation on two types of datasets: 1) Four datasets generated from known Bayesian Networks and 2) Colored images of digits from the MNIST dataset.

\textbf{Bayesian Networks.}
To avoid errors in learning causal structure from data, we perform evaluation on datasets for which the causal structure and the true conditional probabilities of the variables are known from prior research. We select $4$ Bayesian network datasets--- Child, Sachs, Alarm and Water that range from 178-10k parameters (Table~\ref{tbl:dataset})\footnote{www.bnlearn.com/bnrepository}.
Nodes represent the number of input features and arcs denote the causal connections between these features in the network. Each causal connection is specified using a conditional probability table $\tt P(X_i|Parents(X_i))$; we consider these probability values as the parameters in our models. 
To create a prediction task, we select a variable in each of these networks as the output $Y$. The number of classes in Table~\ref{tbl:dataset} denote  the possible values for an output variable. For example, the variable BP (blood pressure) in the alarm dataset takes 3 values i.e, \texttt{LOW, NORMAL, HIGH}. The causal model uses only parents of $Y$ whereas the associational model (DNN) uses all  nodes except $Y$ as features. 

\textbf{Colored MNIST Dataset.}
\newtext{We also evaluate on a complex dataset where it is difficult to construct a causal graph of the input features. For this, we consider the dataset of colored MNIST images used in a recent work by~\cite{arjovsky2019invariant}. The original MNIST dataset consists of grayscale images of handwritten digits (0-9)\footnote{http://yann.lecun.com/exdb/mnist/}. The colored MNIST dataset consists of inputs where digits 0-4 are red in color with label as 0 while 5-9 are green in color and have label 1. The training dataset consists of two environments where only 10\% and 20\% of inputs {\em do not} follow the correlation of color to digits. This creates a spurious correlation of color with the output. In this dataset, {\em shape} of the digit is the actual causal feature whereas {\em color} acts as the associational or non-causal feature. The test dataset is generated such that 90\% of the inputs {\em do not} follow the color pattern. We use the code made available by~\cite{arjovsky2019invariant} to generate the dataset and perform our evaluation\footnote{https://github.com/facebookresearch/InvariantRiskMinimization}.  We refer the readers to the paper for further details.}
\begin{figure*}[t]
    \begin{subfigure}{.33\textwidth}
        \centering
        \includegraphics[scale=0.30]{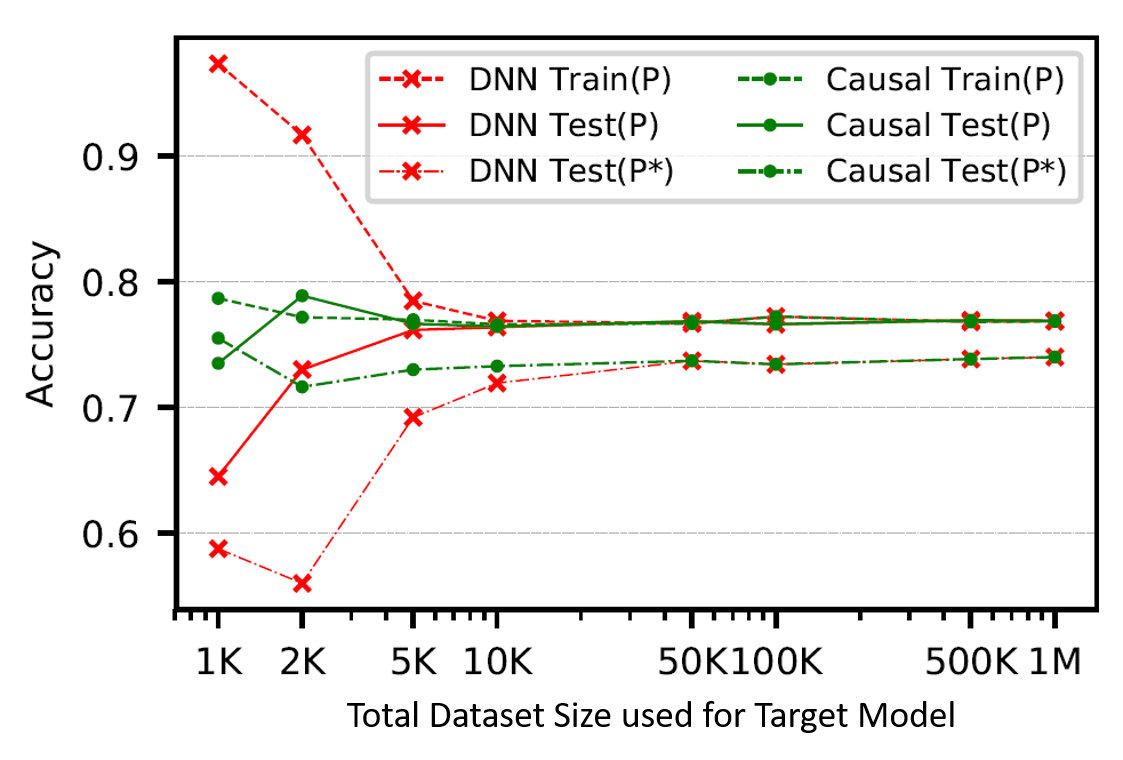}

	\caption{ }
        \label{fig:child_orig}
    \end{subfigure}%
    \hfill
    \begin{subfigure}{0.33\textwidth}
        \centering
        \includegraphics[scale=0.31]{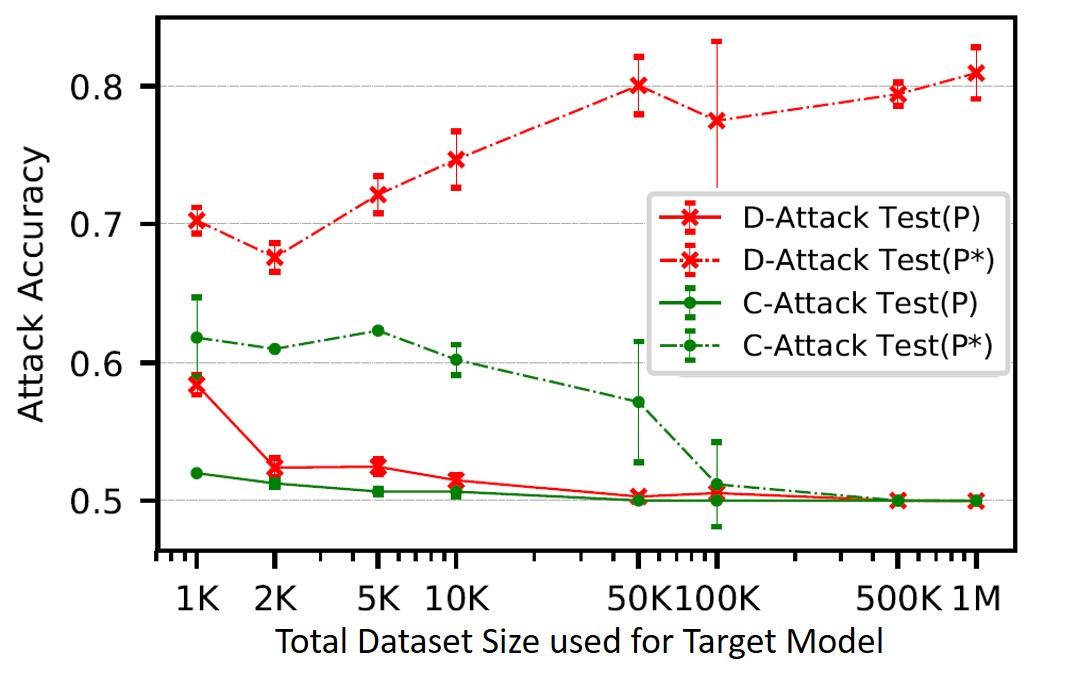}

        \caption{ }
        \label{fig:child_attack}
    \end{subfigure}
    \hfill
     \begin{subfigure}{0.33\textwidth}
        \centering
        \includegraphics[scale=0.21]{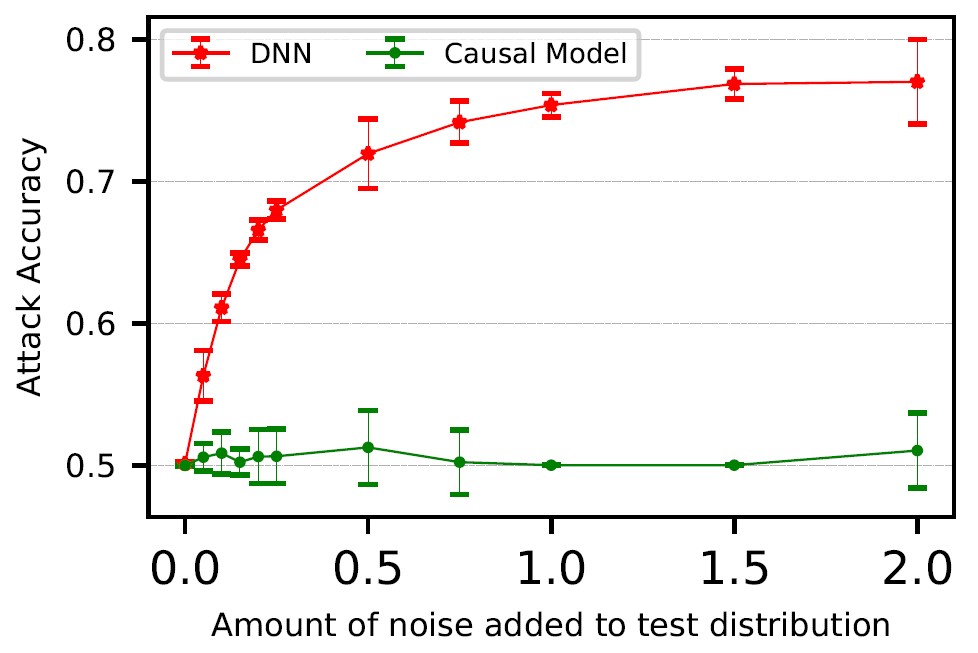}
       \caption{}
        \label{fig:all_dist}
    \end{subfigure}
\vspace{-0.3cm}
\caption{ Results for Child dataset with XrayReport as the output. (~\subref{fig:child_orig}) is the target model accuracy. (~\subref{fig:child_attack}) is the  attack accuracy for different dataset sizes on which the target model is trained and (~\subref{fig:all_dist}) is the attack accuracy for test distribution with varying amount of noise for total dataset size of 100K samples.}
\label{fig:acc_comparison}
\vspace{-0.5cm}
\end{figure*}

\begin{figure*}[t]
  \begin{subfigure}[b]{0.32\textwidth}
        \centering

        \includegraphics[scale=0.3]{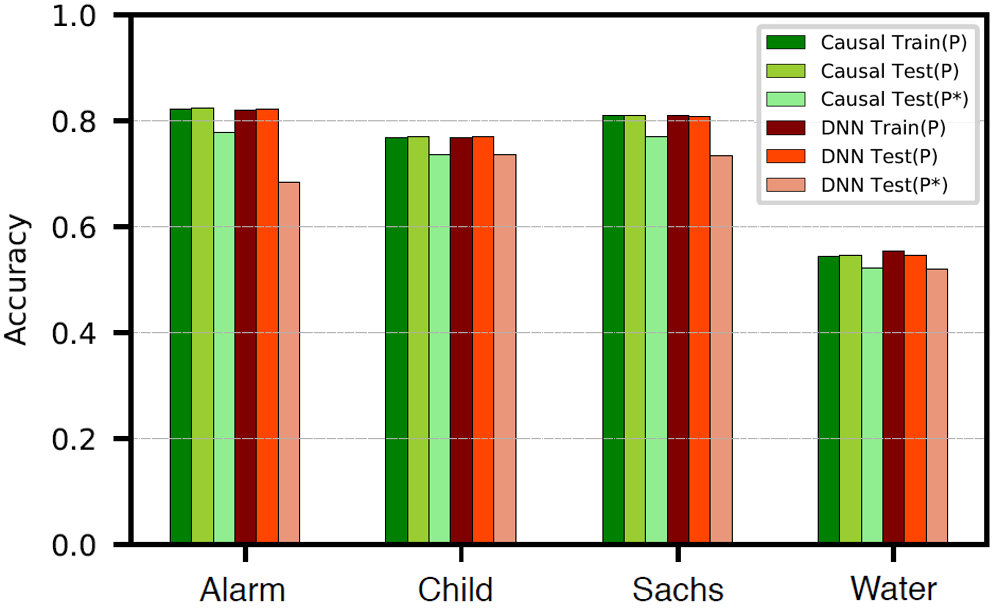}

   \caption{}
\label{fig:orig_all}
  \end{subfigure}
  \hfill
  \begin{subfigure}[b]{0.32\textwidth}
        \centering
        \includegraphics[scale=0.28]{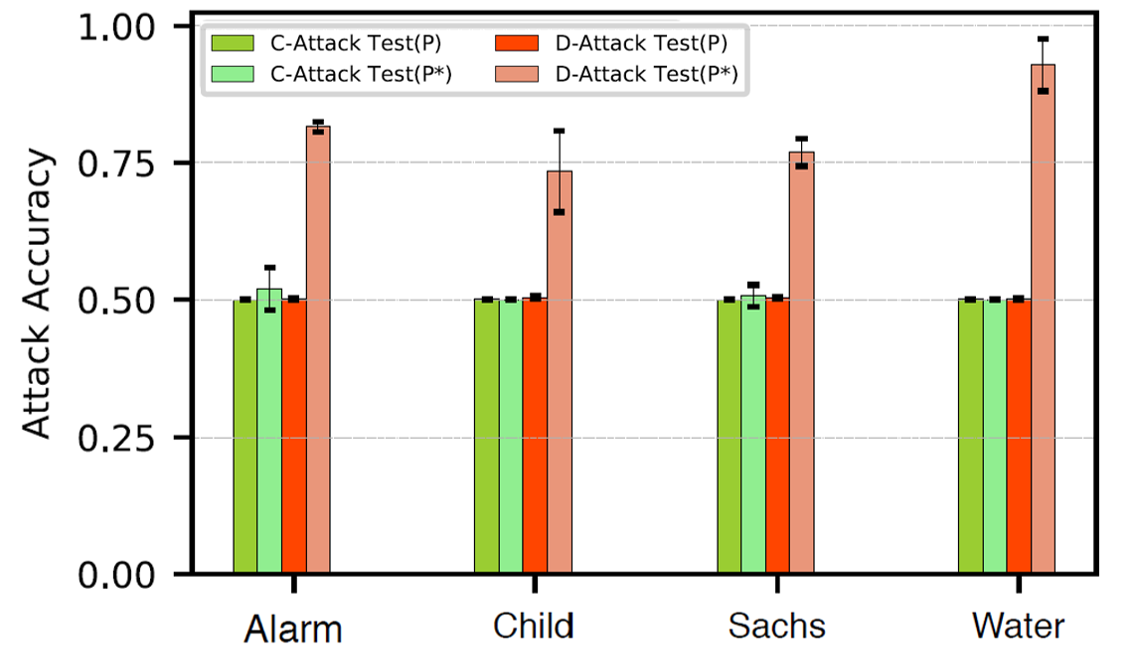}
   \caption{}
\label{fig:all_attack}
\end{subfigure}
\hfill
  \begin{subfigure}[b]{0.32\textwidth}
\centering
        \includegraphics[scale=0.21]{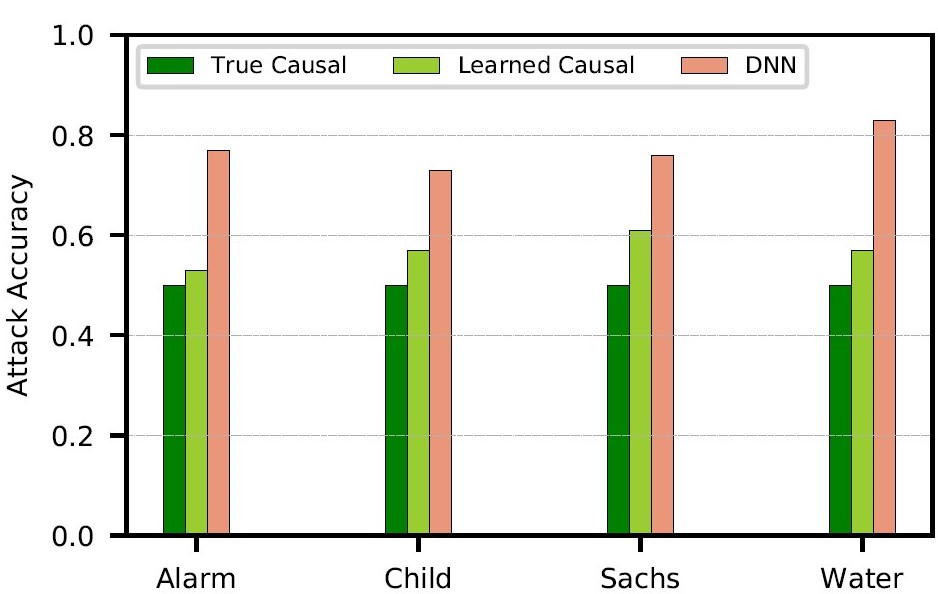}
   \caption{}
\label{fig:learned_graph}
  \end{subfigure}
\vspace{-0.2cm}
\caption{Results for all the bayesian models trained on dataset of size of 60K. (\subref{fig:orig_all}) is the accuracy of the target model,   (\subref{fig:all_attack}) is the attack accuracy for the target model, (\subref{fig:learned_graph}) is the attack accuracy using Test(P*) dataset on true causal, learned causal  and DNN models.}
\label{fig:all_results}
\vspace{-0.5cm}
  \end{figure*}
\subsection{ Results for Bayesian Networks Dataset}

\paragraph{Evaluation Methodology.}
We sample data using the causal structure and probabilities from the Bayesian network, and use a 60:40\% split for train-test datasets. We learn a causal model and a deep neural network (DNN) on each training dataset. 
We implement the attacker model to perform membership inference attack using the output confidences of both these models, based on past work~\citep{salem2018ml}. The input features for the attacker model comprises of the output confidences from the target model, and the output is membership prediction (member / non-member) in the training dataset of the target model. In both the train and the test data for the attacker model, the number of members and non-members are equal. The creation of the attacker dataset is described in Figure~\ref{fig:data_dist} in the \app~\ref{app:exp}. 
Note that the attack accuracies reported are an upper bound since we assume that the adversary has access to the subset of training data for the ML model.

To train the causal model, we use the bnlearn library in R language that supports maximum likelihood estimation of the parameters in $Y$'s conditional probability table. 
For prediction, we use the \texttt{parents} method to predict the class of any specific variable.
To train the DNN model and the attacker model, we build custom estimators in Python using  Tensorflow v1.2. The DNN model is a multilayer perceptron (MLP) with 3 hidden layers of 128, 512 and 128 nodes respectively.  The learning rate is set to 0.0001 and the model is trained for 10000 steps. The attacker model has 2 hidden layers with 5 nodes each, a learning rate of 0.001, and is trained for 5000 steps. Both models use  Adam optimizer,  ReLU for the activation function, and cross entropy as the loss function. We chose these parameters to  ensure model convergence.
We evaluate the DNN and the causal model sample sizes ranging from 1K to 1M dataset sizes. 
We refer Test(P) as the test dataset which is drawn from the same distribution as the training data and Test(P*)  is generated from a completely different distribution except for the relationship of the output class to its parents. To generate Test(P*), we alter the true probabilities $\tt \Pr(X)$ uniformly at random (later, we consider adding noise to the original value).
Our goal with generating Test (P*) is to  capture extreme shifts in distribution for 
input features. 
The causal and DNN model are the \emph{target} on which the  attack is perpetrated. 

\textbf{Accuracy comparison of DNN and Causal models.}
Figure~\ref{fig:child_orig} shows the target model accuracy comparison for the DNN and the causal model trained on the Child dataset with XrayReport as the output variable. We report the accuracy of the target models only for a single run since in practice the attacker would have access to the outputs of only a single model.  
We observe that the DNN model has a large difference between the  train and the test accuracy (both Test(P) and Test(P*)) for smaller dataset sizes (1K and 2K). This indicates that the model overfits on the training data for these dataset sizes. However, after $10$K samples, the model converges such that the train and Test(P) dataset have the same accuracy. The accuracy for the Test(P*) distribution stabilizes for a total dataset size of $10$K samples. 
In contrast, for the causal model, the train and Test(P) accuracy are similar for the causal model even on smaller dataset sizes. However, after convergence at around $10$K samples, the gap between the accuracy of train and Test(P*) dataset is  the same for both the DNN and the causal model.  
Figure~\ref{fig:orig_all} shows similar results for the accuracy  on all the datasets.
 
\textbf{Attack Accuracy of DNN and Causal models.}
A naive attacker classifier would predict all the samples to be members and therefore achieve 0.5 prediction accuracy. Thus, we consider 0.5 as the baseline attack accuracy which is equal to a random guess.
Figure~\ref{fig:child_attack} shows the attack accuracy comparison for Test(P) (same distribution) and Test(P*) (different distribution) datasets. Attack accuracy of  the Test(P) dataset for the causal model is slightly above a random guess for smaller dataset sizes, and then converges to 0.5. In comparison, attack accuracy for the DNN on Test(P) dataset is over 0.6 for smaller samples sizes and reaches 0.5 after 10K datapoints.  
This confirms  past work that an overfitted DNN is susceptible to membership inference attacks even for test data generated from the same distribution as the training data~\citep{yeom2018privacy}. 
On Test(P*), the attack accuracy is always higher for the DNN than the causal model, indicating our main result that associational models ``overfit'' to the training distribution, in addition to the training dataset. Membership inference accuracy for DNNs is as high as 0.8 for total dataset size of 50K while that of causal models is below 0.6.  Further, attack accuracy for DNN increases with sample size whereas  attack accuracy for the causal model reduces to 0.5 for total dataset size over 100k even when the gap between the train and test accuracies is the same as DNNs (Figure~\ref{fig:child_orig}). These results show that causal models generalize better than DNNs across input distributions.   
Figure~\ref{fig:all_attack}  shows a similar result for all four datasets. The attack accuracy for DNNs and the causal model is close to 0.5 for the Test 1 dataset  while for the Test(P*) dataset the attack accuracy is significantly higher for DNNs than causal model. 
This empirically confirms our claim that in general, causal models are robust to membership inference attacks across test distributions as compared to associational models.

\textbf{Attack Accuracy for Different Test Distributions.}
To understand the change in attack accuracy as  $\Pr(X)$ changes, we generate test data from different distributions by adding varying amount of noise to the true probabilities. We range the noise value between 0 to 2 and add it to the individual probabilities which are then normalized to sum up to 1. Figure~(\ref{fig:all_dist}) shows the  attack accuracy for the causal model and the DNN on the child dataset for a total sample size of 100K samples. We observe that the attack accuracy increases with increase in the noise values for the DNN. Even for a small amount of noise, attack accuracies increase sharply. In contrast, attack accuracies stay close to 0.5 for the causal model, demonstrating the robustness to membership attacks.

\textbf{Results with learnt causal model.} Finally, we perform experiments to understand the effect of privacy guarantees on causal structures learned from data that might vary from the true causal structure. 
For these datasets, a simple hill-climbing algorithm returned the true causal parents. Hence we evaluated attack accuracy for models with hand-crafted errors in  learning the structure, i.e., misestimation of causal parents, see Figure~(\ref{fig:learned_graph}). Specifically, we include two non-causal features as parents of the output variable along with the true causal features. The attack risk increases as a learnt model deviates from the true causal structure, however it still exhibits lower attack accuracy than the corresponding associational model. 
Table~\ref{tbl:extra_data} shows the attack and prediction accuracy for  \texttt{Sachs} dataset  when trained with increase in error in the causal model (with 1 and 2 non-causal features), and the results for the corresponding DNN model.
 
\begin{table}[t]
\centering
\footnotesize
\begin{tabular}{@{}c|c|c|c|c@{}}
\toprule
   & True Model                                                 & \multicolumn{2}{c|}{Learned Causal (2 causal +)}                                                                                                                       & DNN \\ \midrule
  \begin{tabular}[c]{@{}c@{}}Acc. \\ (\%)\end{tabular}         & \begin{tabular}[c]{@{}c@{}}2 causal \\ parents\end{tabular} & \begin{tabular}[c]{@{}c@{}}1 non-causal\\  parent\end{tabular} & \begin{tabular}[c]{@{}c@{}}2 non-causal\\ parents\end{tabular} &               \\ \midrule
  Attack  & \textbf{50}                                                          & \textbf{52}                                                                          & \textbf{61}                                                                         & \textbf{76}            \\ \midrule
Pred. & 79                                                          & 75                                                                          & 68.8                                                                        & 73            \\ \bottomrule
\end{tabular}
\caption{{Attack and Prediction accuracy comparison across models for \texttt{Sachs} dataset and \texttt{Akt} output variable.}}
\label{tbl:extra_data}
\end{table}

\subsection{Results for Colored MNIST Dataset}
\newtext{In recent work \citet{arjovsky2019invariant} proposed a way to train a causal model by minimizing the risk across different environments or distributions of the dataset. Using this approach, we train an invariant risk minimizer (IRM)  and an emprical risk minimizer (ERM) on the colored MNIST data. Since IRM constructs the same model using invariant feature representation for the two training domains, it is aimed to learn the causal features (shape) that are also invariant across domains~\cite{peters2016causal}. Thus  IRM can be considered as a causal model while ERM is an associational model. Table~\ref{tbl:colored_mnist} gives the model accuracy and the attack accuracy for IRM and ERM models. The attacker model has 2 hidden layers with 3 nodes each, a learning rate of 0.001, and is trained for 5000 steps. We observe that the causal model has attack accuracy close to a random guess while the associational model has 66\% attack accuracy. Although the training accuracy of IRM is lower than ERM, we expect this to be an acceptable trade-off for the stronger privacy and better generalizability guarantees of causal models.} 

\begin{table}[t]
\centering
	\begin{adjustbox}{width=\columnwidth,center}
		\begin{tabular}{@{} c|c c c@{}}
		\toprule
		{Model}                                                   & \begin{tabular}[c]{@{}c@{}}  Train \\Acc. (\%)\end{tabular}  & \begin{tabular}[c]{@{}c@{}} Test \\ Acc. (\%)\end{tabular}                                            &  \begin{tabular}[c]{@{}c@{}}  Attack \\ Acc. (\%)\end{tabular} \\ \midrule
		\begin{tabular}[c]{@{}c@{}} IRM  (causal)\end{tabular} 						& 70   		  &  69								 & 53                                                                                \\ 
       \begin{tabular}[c]{@{}c@{}} ERM (Associational)\end{tabular}                                           & 87             & 16                                                       & 66                                                                                           \\ \bottomrule
		\end{tabular}
		\end{adjustbox}
		\vspace{-0.2cm}
		      \caption{{Results on Colored MNIST Dataset. }}
\label{tbl:colored_mnist}
		\vspace{-0.1cm}
\end{table}

\section{Related Work}
\label{related}

\textbf{Privacy attacks and defenses on ML models.}
~\citet{shokri2017membership} demonstrate the first membership inference attacks on black box neural network models with access only to the confidence values. Similar attacks have been shown on several other models such as GANs~\citep{hayes2017logan}, text prediction generative models~\citep{carlini2018secret,song2018the} and federated learning models~\citep{nasr2018machine}. However, prior research does not focus on the severity of these attacks with change in the distribution of the test dataset. We discussed in  Section~\ref{sec:robust-membership} that existing defenses based on regularization~\citep{nasr2018machine} are not practical when models are evaluated on test inputs from different distributions. Another line of defense is to add differentially private noise while training the model. However, the $\epsilon$ values necessary to mitigate membership inference attacks in deep neural networks require addition of large amount of noise that degrades the accuracy of the output model~\citep{rahman2018membership}. Thus, there is a trade-off between privacy and utility when using differential privacy for neural networks.  In contrast, we show that  causal models require lower amount of noise to achieve the same $\epsilon$ differential privacy guarantees and hence retain accuracy closer to the original model. Further, as training sample sizes become sufficiently large (Section~\ref{eval})  causal models are robust to membership inference attacks across distributions.

\textbf{Causal learning and privacy.}
There is substantial literature on learning causal models from data; for a review see \citep{peters2017elements,pearl2009causality}. 
\citet{kusner2015private} proposed a method to privately reveal parameters from a causal learning algorithm, using the framework of differential privacy. Instead of a specific causal algorithm, our focus is on the privacy benefits of causal models for general predictive tasks. While recent work uses causal models to study properties of ML models such as providing explanations~\citep{datta2016algorithmic} or fairness~\citep{kusner2017counterfactual}, 
 the relation of causal learning to model privacy  is yet unexplored.

\section{Conclusion and Future Work}
Our results show that causal learning is a promising approach to train models that are robust to privacy attacks such as membership inference and model inversion. As future work, we aim to investigate 
  privacy guarantees  when the causal features and the relationship between them is not known apriori and   with causal insufficiency and selection bias in the observed data.

\section*{Acknowledgements}
We thank the reviewers for their useful comments on the paper. We thank Olga Ohrimenko, Boris Koepf, Amit Deshpande and Emre Kiciman for helpful discussion and feedback on this work.

\bibliography{privacy-causal}

\begin{thebibliography}{33}
\providecommand{\natexlab}[1]{#1}
\providecommand{\url}[1]{\texttt{#1}}
\expandafter\ifx\csname urlstyle\endcsname\relax
  \providecommand{\doi}[1]{doi: #1}\else
  \providecommand{\doi}{doi: \begingroup \urlstyle{rm}\Url}\fi

\bibitem[Aliferis et~al.(2010)Aliferis, Statnikov, Tsamardinos, Mani, and
  Koutsoukos]{aliferis2010local}
Aliferis, C.~F., Statnikov, A., Tsamardinos, I., Mani, S., and Koutsoukos,
  X.~D.
\newblock Local causal and markov blanket induction for causal discovery and
  feature selection for classification part i: Algorithms and empirical
  evaluation.
\newblock \emph{Journal of Machine Learning Research}, 11\penalty0
  (Jan):\penalty0 171--234, 2010.

\bibitem[Arjovsky et~al.(2019)Arjovsky, Bottou, Gulrajani, and
  Lopez-Paz]{arjovsky2019invariant}
Arjovsky, M., Bottou, L., Gulrajani, I., and Lopez-Paz, D.
\newblock Invariant risk minimization.
\newblock \emph{arXiv preprint arXiv:1907.02893}, 2019.

\bibitem[Bengio et~al.(2019)Bengio, Deleu, Rahaman, Ke, Lachapelle, Bilaniuk,
  Goyal, and Pal]{bengio2019meta}
Bengio, Y., Deleu, T., Rahaman, N., Ke, R., Lachapelle, S., Bilaniuk, O.,
  Goyal, A., and Pal, C.
\newblock A meta-transfer objective for learning to disentangle causal
  mechanisms.
\newblock \emph{arXiv preprint arXiv:1901.10912}, 2019.

\bibitem[Carlini et~al.(2018)Carlini, Liu, Kos, Erlingsson, and
  Song]{carlini2018secret}
Carlini, N., Liu, C., Kos, J., Erlingsson, {\'U}., and Song, D.
\newblock The secret sharer: Measuring unintended neural network memorization
  \& extracting secrets.
\newblock \emph{arXiv preprint arXiv:1802.08232}, 2018.

\bibitem[Datta et~al.(2016)Datta, Sen, and Zick]{datta2016algorithmic}
Datta, A., Sen, S., and Zick, Y.
\newblock Algorithmic transparency via quantitative input influence: Theory and
  experiments with learning systems.
\newblock In \emph{Security and Privacy (SP), 2016 IEEE Symposium on}, pp.\
  598--617. IEEE, 2016.

\bibitem[Dwork et~al.(2014)Dwork, Roth, et~al.]{dwork2014algorithmic}
Dwork, C., Roth, A., et~al.
\newblock The algorithmic foundations of differential privacy.
\newblock \emph{Foundations and Trends{\textregistered} in Theoretical Computer
  Science}, 9\penalty0 (3--4):\penalty0 211--407, 2014.

\bibitem[Esteva et~al.(2019)Esteva, Robicquet, Ramsundar, Kuleshov, DePristo,
  Chou, Cui, Corrado, Thrun, and Dean]{esteva2019guide}
Esteva, A., Robicquet, A., Ramsundar, B., Kuleshov, V., DePristo, M., Chou, K.,
  Cui, C., Corrado, G., Thrun, S., and Dean, J.
\newblock A guide to deep learning in healthcare.
\newblock \emph{Nature medicine}, 25\penalty0 (1):\penalty0 24, 2019.

\bibitem[Fischer \& Krauss(2018)Fischer and Krauss]{fischer2018deep}
Fischer, T. and Krauss, C.
\newblock Deep learning with long short-term memory networks for financial
  market predictions.
\newblock \emph{European Journal of Operational Research}, 270\penalty0
  (2):\penalty0 654--669, 2018.

\bibitem[Fredrikson et~al.(2015)Fredrikson, Jha, and
  Ristenpart]{fredrikson2015model}
Fredrikson, M., Jha, S., and Ristenpart, T.
\newblock Model inversion attacks that exploit confidence information and basic
  countermeasures.
\newblock In \emph{Proceedings of the 22nd ACM SIGSAC Conference on Computer
  and Communications Security}, pp.\  1322--1333. ACM, 2015.

\bibitem[Hamm et~al.(2016)Hamm, Cao, and Belkin]{hamm2016learning-dp}
Hamm, J., Cao, Y., and Belkin, M.
\newblock Learning privately from multiparty data.
\newblock In \emph{International Conference on Machine Learning}, pp.\
  555--563, 2016.

\bibitem[Hayes et~al.(2017)Hayes, Melis, Danezis, and
  De~Cristofaro]{hayes2017logan}
Hayes, J., Melis, L., Danezis, G., and De~Cristofaro, E.
\newblock Logan: Membership inference attacks against generative models.
\newblock \emph{arXiv preprint arXiv:1705.07663}, 2017.

\bibitem[Kusner et~al.(2015)Kusner, Sun, Sridharan, and
  Weinberger]{kusner2015private}
Kusner, M.~J., Sun, Y., Sridharan, K., and Weinberger, K.~Q.
\newblock Private causal inference.
\newblock \emph{arXiv preprint arXiv:1512.05469}, 2015.

\bibitem[Kusner et~al.(2017)Kusner, Loftus, Russell, and
  Silva]{kusner2017counterfactual}
Kusner, M.~J., Loftus, J., Russell, C., and Silva, R.
\newblock Counterfactual fairness.
\newblock In \emph{Advances in Neural Information Processing Systems}, pp.\
  4066--4076, 2017.

\bibitem[Mahajan et~al.(2020)Mahajan, Tople, and Sharma]{mahajan2020}
Mahajan, D., Tople, S., and Sharma, A.
\newblock Domain generalization using causal matching.
\newblock \emph{arXiv preprint arXiv:2006.07500}, 2020.

\bibitem[Mansour et~al.(2009)Mansour, Mohri, and
  Rostamizadeh]{mansour2009domain}
Mansour, Y., Mohri, M., and Rostamizadeh, A.
\newblock Domain adaptation: Learning bounds and algorithms.
\newblock \emph{arXiv preprint arXiv:0902.3430}, 2009.

\bibitem[Mnih et~al.(2013)Mnih, Kavukcuoglu, Silver, Graves, Antonoglou,
  Wierstra, and Riedmiller]{mnih2013playing}
Mnih, V., Kavukcuoglu, K., Silver, D., Graves, A., Antonoglou, I., Wierstra,
  D., and Riedmiller, M.
\newblock Playing atari with deep reinforcement learning.
\newblock \emph{arXiv preprint arXiv:1312.5602}, 2013.

\bibitem[Nabi \& Shpitser(2018)Nabi and Shpitser]{nabi2018fair}
Nabi, R. and Shpitser, I.
\newblock Fair inference on outcomes.
\newblock In \emph{Proceedings of the... AAAI Conference on Artificial
  Intelligence. AAAI Conference on Artificial Intelligence}, volume 2018, pp.\
  1931. NIH Public Access, 2018.

\bibitem[Nasr et~al.(2018{\natexlab{a}})Nasr, Shokri, and
  Houmansadr]{nasr2018comprehensive}
Nasr, M., Shokri, R., and Houmansadr, A.
\newblock Comprehensive privacy analysis of deep learning: Stand-alone and
  federated learning under passive and active white-box inference attacks.
\newblock \emph{arXiv preprint arXiv:1812.00910}, 2018{\natexlab{a}}.

\bibitem[Nasr et~al.(2018{\natexlab{b}})Nasr, Shokri, and
  Houmansadr]{nasr2018machine}
Nasr, M., Shokri, R., and Houmansadr, A.
\newblock Machine learning with membership privacy using adversarial
  regularization.
\newblock In \emph{Proceedings of the 2018 ACM SIGSAC Conference on Computer
  and Communications Security}, pp.\  634--646. ACM, 2018{\natexlab{b}}.

\bibitem[Papernot et~al.(2017)Papernot, Abadi, Erlingsson, Goodfellow, and
  Talwar]{papernot2016teacher-dp}
Papernot, N., Abadi, M., Erlingsson, U., Goodfellow, I., and Talwar, K.
\newblock Semi-supervised knowledge transfer for deep learning from private
  training data.
\newblock In \emph{ICLR}, 2017.

\bibitem[Pearl(2009)]{pearl2009causality}
Pearl, J.
\newblock \emph{Causality}.
\newblock Cambridge university press, 2009.

\bibitem[Pellet \& Elisseeff(2008)Pellet and Elisseeff]{pellet2008markovb}
Pellet, J.-P. and Elisseeff, A.
\newblock Using markov blankets for causal structure learning.
\newblock \emph{Journal of Machine Learning Research}, 9\penalty0
  (Jul):\penalty0 1295--1342, 2008.

\bibitem[Peters et~al.(2016)Peters, B{\"u}hlmann, and
  Meinshausen]{peters2016causal}
Peters, J., B{\"u}hlmann, P., and Meinshausen, N.
\newblock Causal inference by using invariant prediction: identification and
  confidence intervals.
\newblock \emph{Journal of the Royal Statistical Society: Series B (Statistical
  Methodology)}, 78\penalty0 (5):\penalty0 947--1012, 2016.

\bibitem[Peters et~al.(2017)Peters, Janzing, and
  Sch{\"o}lkopf]{peters2017elements}
Peters, J., Janzing, D., and Sch{\"o}lkopf, B.
\newblock \emph{Elements of causal inference: foundations and learning
  algorithms}.
\newblock MIT press, 2017.

\bibitem[Rahman et~al.(2018)Rahman, Rahman, Laganiere, Mohammed, and
  Wang]{rahman2018membership}
Rahman, M.~A., Rahman, T., Laganiere, R., Mohammed, N., and Wang, Y.
\newblock Membership inference attack against differentially private deep
  learning model.
\newblock \emph{Transactions on Data Privacy}, 2018.

\bibitem[Salem et~al.(2018)Salem, Zhang, Humbert, Fritz, and
  Backes]{salem2018ml}
Salem, A., Zhang, Y., Humbert, M., Fritz, M., and Backes, M.
\newblock Ml-leaks: Model and data independent membership inference attacks and
  defenses on machine learning models.
\newblock \emph{arXiv preprint arXiv:1806.01246}, 2018.

\bibitem[Scutari(2009)]{scutari2009bnlearn}
Scutari, M.
\newblock Learning bayesian networks with the bnlearn r package.
\newblock \emph{arXiv preprint arXiv:0908.3817}, 2009.

\bibitem[Shalev-Shwartz \& Ben-David(2014)Shalev-Shwartz and
  Ben-David]{shalevshwartz2014book}
Shalev-Shwartz, S. and Ben-David, S.
\newblock \emph{Understanding Machine Learning: From Theory to Algorithms}.
\newblock Cambridge University Press, 2014.
\newblock \doi{10.1017/CBO9781107298019}.

\bibitem[Shokri et~al.(2017)Shokri, Stronati, Song, and
  Shmatikov]{shokri2017membership}
Shokri, R., Stronati, M., Song, C., and Shmatikov, V.
\newblock Membership inference attacks against machine learning models.
\newblock In \emph{Security and Privacy (SP), 2017 IEEE Symposium on}, pp.\
  3--18. IEEE, 2017.

\bibitem[{Song} \& {Shmatikov}(2018){Song} and {Shmatikov}]{song2018the}
{Song}, C. and {Shmatikov}, V.
\newblock The natural auditor: How to tell if someone used your words to train
  their model.
\newblock \emph{arXiv preprint arXiv:1811.00513}, 2018.

\bibitem[Tsantekidis et~al.(2017)Tsantekidis, Passalis, Tefas, Kanniainen,
  Gabbouj, and Iosifidis]{tsantekidis2017using}
Tsantekidis, A., Passalis, N., Tefas, A., Kanniainen, J., Gabbouj, M., and
  Iosifidis, A.
\newblock Using deep learning to detect price change indications in financial
  markets.
\newblock In \emph{2017 25th European Signal Processing Conference (EUSIPCO)},
  pp.\  2511--2515. IEEE, 2017.

\bibitem[Wu et~al.(2015)Wu, Fredrikson, Wu, Jha, and
  Naughton]{wu2015revisitingdp}
Wu, X., Fredrikson, M., Wu, W., Jha, S., and Naughton, J.~F.
\newblock Revisiting differentially private regression: Lessons from learning
  theory and their consequences.
\newblock \emph{arXiv preprint arXiv:1512.06388}, 2015.

\bibitem[Yeom et~al.(2018)Yeom, Giacomelli, Fredrikson, and
  Jha]{yeom2018privacy}
Yeom, S., Giacomelli, I., Fredrikson, M., and Jha, S.
\newblock Privacy risk in machine learning: Analyzing the connection to
  overfitting.
\newblock In \emph{2018 IEEE 31st Computer Security Foundations Symposium
  (CSF)}, pp.\  268--282. IEEE, 2018.

\end{thebibliography}
\bibliographystyle{icml2020}
\newpage
\clearpage
\appendix
\twocolumn[
\icmltitle{Supplementary Material: Alleviating Privacy Attacks via Causal Learning}
\vskip 0.3in

]
\section{Generalization Properties of Causal Models}
\subsection{Generalization over Different Distributions}
\label{app:gen_thm}

We provide formal proofs for generalization properties of causal model over different distributions and over a single datapoint.

\generalitytheorem*

\begin{proof}
    The proof has three parts: General ODE Bound for a model, equivalence of loss-minimizing causal hypotheses (models) on $\tt P$ and $\tt P^*$,  and finally the two claims from the Theorem. 
    
    \paragraph{I. GENERAL ODE BOUND} ~\\ 
    Consider a model $\tt h:X \rightarrow Y$ belonging to a set of models $\tt \mathcal{H}$,  that was trained on $\tt {S} \sim P(X, Y)$.  From Def.~\ref{def:ode} we  write, 
    \begin{equation}
    \label{eq:one}
    \begin{split}
    \tt ODE_{P,P^*} (h, y) &\tt =   \mathcal{L}_{P^*}(h, y) -  \mathcal{L}_{{S}\sim P}(h, y)  \\
    & \tt =         \mathcal{L}_{P^*}(h, y) -  \mathcal{L}_{P}(h, y) + \\
     & \tt \mathcal{L}_{P}(h, y) - \mathcal{L}_{{S}\sim P}(h, y) \\
     & \tt =         \mathcal{L}_{P^*}(h, y) -  \mathcal{L}_{P}(h, y) +IDE_P(h, y) 
    \end{split}
    \end{equation}
    
    where the last equation is to due to Def.\ref{def:ide} of the in-distribution  generalization error. 

    Let us denote the optimal loss-minimizing hypotheses over $\tt \mathcal{H}$ for $\tt P$ and $\tt P^*$ as $\tt h_{P}^{OPT}$ and $\tt h_{P^*}^{OPT}$.
  \begin{equation}
     \tt h_{P}^{OPT}= \arg \min_{h \in \mathcal H} \mathcal{L}_{P}(h, y)  \qquad
     \tt h_{P*}^{OPT}= \arg \min_{h \in \mathcal H} \mathcal{L}_{P*}(h, y)
  \end{equation}

    Using the triangle inequality of the loss function, we can write: 
    \begin{equation}
     \label{eq:two}
     \tt  \mathcal{L}_{P^*}(h, y) \leq  \mathcal{L}_{P^*}(h, h_{P}^{OPT}) +   \mathcal{L}_{P^*}(h_{P}^{OPT}, y)
    \end{equation}
    
    And, 
    \begin{equation}
    \label{eq:three}
    \begin{split}
      \tt  \mathcal{L}_P(h, y)  & \tt \geq  \mathcal{L}_P(h, h_P^{OPT}) - \mathcal{L}_P(h_P^{OPT}, y)\\
     \tt   \Rightarrow -\mathcal{L}_P(h, y) & \tt \leq  -\mathcal{L}_P(h, h_P^{OPT}) + \mathcal{L}_P(h_P^{OPT}, y)
     \end{split}
    \end{equation}

    Thus, combining Eqns.~\ref{eq:one}, ~\ref{eq:two} and ~\ref{eq:three}, we obtain, 
    \begin{equation}
        \begin{split}
            \tt ODE_{P,P^*}&\tt (h, y) \\
                                       & \tt \leq IDE_P(h, y) +  \mathcal{L}_{P^*}(h, h_P^{OPT}) + \\
     & \tt \mathcal{L}_{P^*}(h_{P}^{OPT}, y) -
     \mathcal{L}_P(h, h_P^{OPT}) + \mathcal{L}_P(h_P^{OPT}, y)\\
    & \tt = IDE_P(h, y) +  (\mathcal{L}_{P^*}(h, h_P^{OPT}) -\mathcal{L}_P(h, h_P^{OPT}) )+ \\
    & \tt \mathcal{L}_{P^*}(h_{P}^{OPT}, y) + \mathcal{L}_P(h_P^{OPT}, f)\\
    & \tt \leq IDE_P(h, y) + disc_{L, \mathcal{H}}(P, P^*)  +\\
    & \tt \mathcal{L}_{P^*}(h_{P}^{OPT}, y) + \mathcal{L}_P(h_P^{OPT}, y) \label{eqn:thm1-loss-eqn}
       \end{split}
    \end{equation}

    where the last inequality is due to the definition of discrepancy distance (Definition~\ref{def:discrepancy}). 
    
    Below we show that Eqn.~\ref{eqn:thm1-loss-eqn} divides the out-of-distribution generalization error of a model $\tt h$  in four parts. As defined in the Theorem statement,  $\tt \mathcal{H_C}$ refers to the class of models that uses all causal features ($\tt X_C$), parents of $\tt Y$ over the structural causal graph; and $\tt \mathcal{H_A}$ refers to the class of associational models that may use all or a subset of all available features. 
    \begin{enumerate} 
        \item $\tt IDE_{P}(h, y)$ denotes the in-distribution error of $\tt h$. This can be bounded by typical generalization bounds, such as the uniform error bound that depends only on the  VC dimension and sample size of $\tt S$~\citep{shalevshwartz2014book}. 
        Using a uniform error bound based on the VC dimension,  we obtain, with probability at least $1-\delta$, 
        \begin{equation} \label{eqn:ide-bound}
            \begin{split}
            \tt IDE \leq \sqrt{8\frac{VCdim(\mathcal{H}) ( \ln(2 |{S}|) + 1 ) + \ln(4/\delta) }{|{S}|}} \\
            \tt = IDE\texttt{-}Bound(\mathcal{H}, {S})
        \end{split}
\end{equation} 
Since $\tt \mathcal{H}_C \subseteq \mathcal{H}_A$, VC-dimension of causal models is not greater than that of associational models. Thus, 
\begin{equation} \label{eqn:ide-vcdim}
    \begin{split}
    \tt    VCDim(\mathcal{H}_C) \leq VCDim(\mathcal{H}_A) \Rightarrow IDE\texttt{-}Bound(\mathcal{H}_C, \mathcal{S}) \\
    \tt \leq IDE\texttt{-}Bound(\mathcal{H}_A, \mathcal{S})
\end{split}
\end{equation}
        \item $\tt disc_{L, \mathcal{H}}(P, P^*)$ denotes the distance between the two distributions.  Given two distributions, the discrepancy distance does not depend on $\tt h$, but only on the model class $\mathcal{H}$. From Definition~\ref{def:discrepancy}, discrepancy distance is the maximum quantity over all pairs of models in a model class. 
            Since $\tt \mathcal{H}_C \subseteq \mathcal{H}_A$, we obtain that:
            \begin{equation} \label{eqn:discrepancy-bound}
               \tt  disc_{L, \mathcal{H_C}}(P, P^*) \leq disc_{L, \mathcal{H_A}}(P, P^*) 
            \end{equation}
        \item   $\tt \mathcal{L}_P(h_P^{OPT}, y)$
            measures the error of the loss-minimizing model on $\tt P$, when evaluated on $\tt P$. While $\tt h_P^{OPT}$ is optimal, there can still be error due to the true labeling function $\tt f$ being outside the model class $\tt \mathcal{H}$, or irreducible error due to probabilistic generation of Y.
        \item $\tt \mathcal{L}_{P^*}(h_{P}^{OPT}, y)$ 
            measures the error of the loss-minimizing model on $\tt P$, when evaluated on $\tt P^*$. In addition to the reasons cited above, this error can be due to differences in both $\tt \Pr(X)$ and $\tt \Pr(Y|X)$ between $\tt P$ and $\tt P^*$: change in the marginal distribution of inputs $\tt X$, and/or change in the conditional distribution of $\tt Y$ given $\tt X$. 
    \end{enumerate}

\paragraph{II. SAME LOSS-MINIMIZING CAUSAL MODEL OVER $\tt P$ AND $\tt P^*$} ~\\
Below we show that for a given distribution $\tt P$ and another distribution $\tt P^*$ such that $\tt P(Y|X_C)=P^*(Y|X_C)$, the loss minimizing  model is the same for causal models ($\tt h_{c,P}^{OPT}=h_{c,P^*}^{OPT}$), but not necessarily for associational models. 

    \paragraph{Causal Model.}
    Given a structural causal network, let us construct a model using all parents of $\tt X_C$ of $\tt Y$.
By property of the structural causal network, $\tt X_C$ includes all parents of $\tt Y$ and therefore 
there are no backdoor paths.  Using Rule 2 of do-calculus from ~\citet{pearl2009causality}:
\begin{equation}
\tt \Pr(Y|do(X_c=x_c))=P(Y|X_C=x_c) = P^*(Y|X_C=x_c)
\end{equation}
where the last equality is assumed since data from $\tt P^*$ also shares the same causal graph.  
Defining $\tt h_{c,P}^{OPT} = \arg \min\limits_{h_c \in \mathcal H_C} \mathcal{L}_{P}(h_c, y)$ and $\tt h_{c,P^*}^{OPT}= \arg \min\limits_{h_c \in \mathcal{H_C}} \mathcal{L}_{P^*}(h_c, y)$,  
 we can write, 
\begin{equation}
\begin{split}
    \tt h_{c,P}^{OPT} & \tt =\arg \min_{h \in \mathcal H_C} \mathcal{L}_{P}(h, y)  \\
    & \tt = \arg \min_{h \in \mathcal H_C} \mathbb{E}_{P(x_c,y)}L(h(x_c), y)=f_{P(Y|X_C)}
    \end{split}
\end{equation}
since $\tt f = \arg \min_h L_x(h(x_c),y)$ for all $\tt x_c$ and thus does not depend on $\tt \Pr(X_C)$, and $\tt f\in \mathcal{H}_C$. Similarly, for $\tt h_{c,P^*}^{OPT}$, we can write:
\begin{equation}
\begin{split}
    \tt h_{c,P^*}^{OPT} & \tt =\arg \min_{h \in \mathcal H_C} \mathcal{L}_{P^*}(h, y)  \\
    & \tt = \arg \min_{h \in \mathcal H_C} \mathbb{E}_{P^*(x_c,y)}L(h(x_c), y)=f_{P^*(Y|X_C)}
        \end{split}
\end{equation}

Since $\tt P(Y|X_C) = P^*(Y|X_C)$, we obtain,
\begin{equation}
\label{eq:equal_h}
\tt f_{P(Y|X_C)}= f_{P^*(Y|X_C)} \Rightarrow  h_{c,P}^{OPT}= h_{c,P^*}^{OPT} 
 \end{equation}

 \paragraph{Associational Model.}In contrast, an associational model may use a subset $\tt X_A \subseteq X$ that may not include all parents of $\tt Y$, or may include parents but also include other extraneous variables. 
Following the derivation for causal models, let us define $\tt h_{a,P}^{OPT} = \arg \min\limits_{h_a \in \mathcal H_A} \mathcal{L}_{P}(h_a, y)$ and $\tt h_{a,P^*}^{OPT}= \arg \min\limits_{h_a \in \mathcal{H_A}} \mathcal{L}_{P^*}(h_a, y)$,  
 we can write, 
\begin{equation}
\begin{split}
    \tt h_{a,P}^{OPT} & \tt =\arg \min_{h \in \mathcal H_A} \mathcal{L}_{P}(h, y) \\
    & \tt = \arg \min_{h \in \mathcal H_A} \mathbb{E}_{P(x_a,y)}L(h(x_a), y)=f_{P(X_A, Y)}
    \end{split}
\end{equation}
where we define $\tt f_A$ as, $\tt f_A = \arg \min_h L_x(h(x_a),y)$ for any $\tt x_a$. Similarly, for $\tt h_{a,P^*}^{OPT}$, we can write:
\begin{equation}
\begin{split}
    \tt h_{a,P^*}^{OPT} & \tt =\arg \min_{h \in \mathcal H_A} \mathcal{L}_{P^*}(h, y) \\
    & \tt = \arg \min_{h \in \mathcal H_A} \mathbb{E}_{P^*(x_a,y)}L(h(x_a), y)=f_{P^*(X_A, Y)}
    \end{split}
\end{equation} 

Now, in general,  
$$\tt {P(X_A, Y)} \neq {P^*(X_A, Y)} \Rightarrow f_{P(X_A, Y)} \neq f_{P^*(X_A, Y)}$$ 
Even if the optimal associational model $\tt f_A \in \mathcal{H}_A$ (as we assumed for causal models), and thus $ \tt f_{P(X_A, Y)} = f_{P( Y|X_A)}$ and $\tt f_{P^*(X_A, Y)} = f_{P^*(Y|X_A)}$, they are not the same    since $\tt P(Y|X_A) \neq P^*(Y|X_A)$. Therefore we obtain,
\begin{equation}
\label{eq:notequal_h}
\tt f_{P(Y|X_A)} \neq f_{P^*(Y|X_A)} \Rightarrow  h_{a,P}^{OPT} \neq h_{a,P^*}^{OPT} 
 \end{equation}

 
 That said, since $\tt X_C \subset X$, it is possible that $\tt X_A=X_C$ for some $\tt X$ and $\tt \mathcal{H}$,  and thus the loss-minimizing associational model includes only the causal features of $\tt Y$. Then $\tt h_{a,P}^{OPT} = h_{a,P^*}^{OPT}$. 
 In general, though, $\tt h_{a,P}^{OPT} \neq h_{a,P^*}^{OPT}$. 


\paragraph{IIIa. CLAIM 1} ~\\
As a warmup, consider the case when $\tt Y$ is generated deterministically. That is, the optimal model $\tt f$ has zero error. Then, both the loss-minimizing causal model and loss-minimizing associational model have zero error when evaluated on the same distribution that they were trained on. Thus,    $\tt \mathcal{L}_{P}(h_{c,P}^{OPT}, y)=\mathcal{L}_{P^*}(h_{c,P^*}^{OPT}, y)= 0$. Similarly,  $\tt \mathcal{L}_{P}(h_{a,P}^{OPT},y)=0$. (Note that here we consider only those cases where $\tt f_{P(Y|X)} \in \mathcal H_A$ and $\tt f_{P^*(Y|X)} \in \mathcal H_A$ for a fair comparison; otherwise, the error bound for $\tt h_a \in \mathcal H_A$ is trivially larger than that for $\tt h_c \in \mathcal H_C$).

Further, for a causal model, using Equation~\ref{eq:equal_h}, we obtain:
\begin{equation} \label{eq:equal-loss-on-pstar}
    \tt \mathcal{L}_{P^*}(h_{c,P}^{OPT}, y) =\mathcal{L}_{P^*}(h_{c,P^*}^{OPT}, y) = 0
\end{equation}
However, the same does not hold for associational models: $\tt \mathcal{L}_{P^*}(h_{a,P}^{OPT}, y)$ need not be zero.

We now present the loss bounds. Using Equations~\ref{eq:equal_h} and \ref{eq:equal-loss-on-pstar},  we write Equation~\ref{eqn:thm1-loss-eqn} for a causal model as:
\begin{equation}
\label{eq:final_1}
\begin{split}
        \tt ODE_{P,P^*}(h_c, y)   \tt = \mathcal{L}_{P^*}(h_c, y) -  \mathcal{L}_{{S}\sim P}(h_c, y) \\
                                  \tt \leq disc_{L, \mathcal{H}_C}(P, P^*) + IDE_P(h_c, y) \\
\end{split}
\end{equation}

For an associational model, we obtain,
\begin{equation}
\label{eq:final_2}
\begin{split}
        \tt ODE_{P,P^*}(h_a, y)  & \tt = \mathcal{L}_{P^*}(h_a, y) -  \mathcal{L}_{{S}\sim P}(h_a, y)\\
                                 & \tt \leq disc_{L, \mathcal{H}_A}(P, P^*) + IDE_P(h_a, y)   \\
                                 & \tt + \mathcal{L}_{P^*} (h_{a,P}^{OPT}, y) \\
 \end{split}
\end{equation}

Using Eqn.~\ref{eqn:ide-bound} that bounds IDE with probability $1-\delta$, and  Eqns.~\ref{eqn:ide-vcdim} and \ref{eqn:discrepancy-bound} that compare IDE-Bound and discrepancy distance between  causal and associational model classes, we  can rewrite Eqn.~\ref{eq:final_1}. With probability at least $1-\delta$: 
\begin{align} 
    \tt ODE_{P,P^*}(h_c, y )  
                                       & \tt \leq disc_{L, \mathcal{H_C}}(P, P^*) + IDE\texttt{-}Bound_P(\mathcal{H}_C,S; \delta) \nonumber \\
                                       & \tt = ODE\texttt{-}Bound_{P, P^*}(h_c, y; \delta) \nonumber \\
                                       & \tt \leq disc_{L, \mathcal{H_A}}(P, P^*) + IDE\texttt{-}Bound_P(\mathcal{H}_A,  S; \delta)  
                                      \label{eq:final_11}
\end{align}

Similarly, for the associational model, 
\begin{equation}
\label{eq:final_22}
\begin{split}
        \tt ODE_{P,P^*}(h_a, y)  & \tt \leq disc_{L, \mathcal{H}_A}(P, P^*) + IDE\texttt{-}Bound_P(\mathcal{H}_A, {S}; \delta)  \\
                                 & \tt + \mathcal{L}_{P^*} (h_{a,P}^{OPT}, y) \\
                                      & \tt = ODE\texttt{-}Bound_{P, P^*}(h_a, y; \delta)
 \end{split}
\end{equation}

Therefore, comparing Eqn. ~\ref{eq:final_11} and ~\ref{eq:final_22},    we claim for any $\tt P$ and $\tt P^*$, with probability $(1-\delta)^2$,
\begin{equation}
\tt ODE\texttt{-}Bound_{P,P^*}(h_c, y; \delta) \leq  ODE\texttt{-}Bound_{P,P^*}(h_a, y; \delta)
\end{equation}

\paragraph{IIIb. CLAIM 2} ~\\
We now consider the general case when Y is generated probabilistically. Thus, even though $\tt f \in \mathcal{H}_C$ and $\tt h_{c,P}^{OPT}= h_{c,P^*}^{OPT} = f$, $\tt \mathcal{L}_{P}(h_{c,P}^{OPT},y) \neq 0$ and $\tt \mathcal{L}_{P^*}(h_{c,P^*}^{OPT}, y) \neq 0$.

Using the IDE bound from Eqn.~\ref{eqn:ide-bound},  we write Eqn.~\ref{eqn:thm1-loss-eqn} as,
\begin{align}
    \tt ODE_{P,P^*}(h_c, y ) 
                                       & \tt \leq disc_{L, \mathcal{H}_C}(P, P^*) + IDE_P(h_c, y) \nonumber \\
                                       & \tt + \mathcal{L}_{P^*}(h_{c,P}^{OPT}, y) +  \mathcal{L}_{P}(h_{c,P}^{OPT},y) \nonumber \\
                                       & \tt \leq disc_{L, \mathcal{H_C}}(P, P^*) + IDE\texttt{-}Bound_P(\mathcal{H}_C,  {S}; \delta) \nonumber \\
                                       & \tt + \mathcal{L}_{P^*}(h_{c,P^*}^{OPT}, y) +  \mathcal{L}_{P}(h_{c,P}^{OPT},y) \label{eqn:claim2-equal-h}\\
                                      & \tt = ODE\texttt{-}Bound_{P, P^*}(h_c, y; \delta) \nonumber \\
                                       & \tt \leq disc_{L, \mathcal{H_A}}(P, P^*) + IDE\texttt{-}Bound_P(\mathcal{H}_A,  {S}; \delta) \nonumber \\
                                       & \tt + \mathcal{L}_{P^*}(h_{c,P^*}^{OPT}, y) +  \mathcal{L}_{P}(h_{c,P}^{OPT},y) \label{eqn:ode-prob-causal}
\end{align}
where Eqn.~\ref{eqn:claim2-equal-h} uses $\tt h_{c,P}^{OPT}=h_{c,P^*}^{OPT}$ and Eqn.~\ref{eqn:ode-prob-causal} uses inequalities comparing IDE and discrepancy distance from Eqns.~\ref{eqn:ide-vcdim} and \ref{eqn:discrepancy-bound}.

Similarly, for associational model, 
\begin{align}
    \tt ODE\texttt{-}&\tt Bound_{P,P^*}(h_a, y )  
                                        \tt = disc_{L, \mathcal{H_A}}(P, P^*) \nonumber\\
                     &\tt + IDE\texttt{-}Bound_P(\mathcal{H}_A,  {S}; \delta)  
                                        \tt + \mathcal{L}_{P^*}(h_{a,P}^{OPT}, y) +
                                        \mathcal{L}_P(h_{a,P}^{OPT}, y) \label{eqn:ode-prob-assoc} 
\end{align}

Now, we compare the last two terms of Equations~\ref{eqn:ode-prob-causal} and~\ref{eqn:ode-prob-assoc}. Since $ \tt \mathcal{H_C} \subseteq \mathcal{H_A}$, loss of the loss-minimizing associational model can be lower than the loss of the causal model trained on the same distribution. Thus, $\tt \mathcal{L}_P(h_{a,P}^{OPT}, y) \leq  \mathcal{L}_P(h_{c,P}^{OPT}, y)$. 

However, since $\tt h_{a,P}^{OPT} \neq h_{a,P^*}^{OPT}$, loss of the loss-minimizing associational model trained on $\tt P$ can  be higher on $\tt P^*$ than the loss of optimal causal model trained on $\tt P^*$ and evaluated on $\tt P^*$. Formally,  let $\tt \gamma_1 \geq 0 $ be the loss reduction over $\tt P$ due to use of associational model optimized on $\tt P$, compared to the loss-minimizing causal model. Similarly, let $\tt \gamma_2$ be the increase in loss over $\tt P^*$ due to using the associational model optimized over $\tt P$, compared to the loss-minimizing causal model.
\begin{align}
    \tt \gamma_1 =& \tt  \mathcal{L}_P(h_{c,P}^{OPT}, y) - \mathcal{L}_P(h_{a,P}^{OPT}, y) \\
    \tt \gamma_2 =& \tt  \mathcal{L}_{P^*}(h_{a,P}^{OPT}, y) - \mathcal{L}_{P^*}(h_{c,P}^{OPT}, y) 
\end{align}
Then, Eqn.~\ref{eqn:ode-prob-assoc} transforms to, 
\begin{align}
    \tt ODE_{P,P^*}(h_a, y )  
                                       & \tt \leq disc_{L, \mathcal{H_A}}(P, P^*) + IDE\texttt{-}Bound_P(\mathcal{H}_A,  \mathcal{S}; \delta) \nonumber \\ 
                                       & \tt + \mathcal{L}_{P^*}(h_{c,P^*}^{OPT}, y) +
                                       \mathcal{L}_P(h_{c,P}^{OPT}, y)  + \gamma_2 - \gamma_1
\end{align}
Hence, as long as $\tt \gamma_2 \geq \gamma_1$, we obtain, 
\begin{equation}
    \tt ODE\texttt{-}Bound_{P,P^*}(h_c, y; \delta) \leq ODE\texttt{-}Bound_{P,P^*}(h_a, y; \delta )
\end{equation} 
Below we show that such a $\tt P^*$ always exists, and further, the worst-case $\tt \max_{P^*} ODE\textnormal{-}Bound_{P, P^*}(h, y; \delta)$ is always lower for a causal model than an associational model.  

\paragraph{There exists $\tt P^*$ such that $\gamma_2 \geq \gamma_1$.}
The proof is by construction. As an example, consider L1 loss and a distribution $\tt P$ such that the optimal causal model $f$ for an input data point $\tt x^{(i)}$ can be written as,
\begin{align} 
    \tt y^{(i)} &\tt = f_P(x_C^{(i)}) + \xi_i = f_{P^*}(x_C^{(i)}) + \xi_i \label{eqn:proof-example-orig} 
\end{align}
where $\tt f(x_C)=h_{c,P}^{OPT}=h_{c,P^*}^{OPT}$ refers to the optimal causal model and is the same for $\tt P$ and $\tt P^*$ (using Eqn.~\ref{eq:equal_h}). Let $\tt f_P(x_A)=h_{a,P}^{OPT}$ be the optimal associational model over $\tt P$. We can rewrite $\tt h_{a,P}^{OPT}$ as an arbitrary change from $\tt h_{c,P}^{OPT}$, using $\tt \lambda_{x_A}^{(i)}$ as a parameter that can be different for each data point $\tt x^{(i)}$. That is,
\begin{equation}
    \tt h_{a,P}^{OPT}(x^{(i)}) = h_{c,P}^{OPT}(x_C^{(i)}) + \lambda_{x_A}^{(i)} \label{eqn:proof-example-assoc}
\end{equation}
Based on Eqns.~\ref{eqn:proof-example-orig} and \ref{eqn:proof-example-assoc}, $\gamma_1$ can be written as,
\begin{equation}
    \begin{split}
        \tt \mathcal{L}_P(h_{c,P}^{OPT}, y)& \tt = \mathbb{E}_P[|\xi|] \\
        \tt \mathcal{L}_P(h_{a,P}^{OPT}, y) &\tt = \mathbb{E}_P[|\xi - \lambda_{x_A}|] \\
        \tt \Rightarrow \gamma_1 & \tt = \mathbb{E}_P[|\lambda_{x_A}|] 
    \end{split}
\end{equation}
Then, we can construct a $\tt P^*(X, Y)$ such that (i) the relationship ($\tt \Pr(Y|X_A)$) between $\tt x_A$ and $\tt y$ is reversed, and (ii) $\tt \Pr(X)$ is chosen such that $\tt \mathbb{E}_{P^*}[\lambda_{x_A}] \geq \mathbb{E}_P[\lambda_{x_A}]$ (e.g., by assigning higher probability weights to data points $\tt i$ where $\tt |\lambda_{x_A}^{(i)}|$ is high). That is, consider a $\tt P^*$ such that  we can write $\tt h_{a,P^*}^{OPT}$ as,
\begin{equation}
    \tt h_{a,P^*}^{OPT}(x^{(i)}) = h_{c,P^*}^{OPT}(x_C^{(i)}) - \lambda_{x_A}^{(i)}
\end{equation}
On such $\tt P^*$, the loss-minimizing causal model remains the same. However, the loss of the associational model $\tt h_{a,P}^{OPT}$ on such $\tt P^*$ increases and can be written as: 
\begin{equation}
    \begin{split}
        \tt \mathcal{L}_{P^*}(h_{c,P}^{OPT}, y)& \tt = \mathbb{E}_{P^*}[|\xi|] \\ 
    \tt \mathcal{L}_{P^*}(h_{a,P}^{OPT}, y) & \tt = \mathbb{E}_{P^*}[|\xi + \lambda_{x_A}|] \\ 
        \tt \Rightarrow \gamma_2 & \tt = \mathbb{E}_{P^*}[|\lambda_{x_A}|] 
    \end{split}
\end{equation}
From condition (ii) above, $\tt \mathbb{E}_{P^*}[\lambda_{x_A}] \geq \mathbb{E}_P[\lambda_{x_A}]$, thus $\tt \gamma_2 \geq \gamma_1$. 

Note that we did not use any special property of the L1 Loss above. In general, we can write the loss-minimizing  function $\tt h_{a,P}^{OPT}$ as adding some arbitrary value $\tt \lambda_{x_A}^{(i)}$ to $\tt h_{c,P}^{OPT}(x_c^{(i)})$; and then construct a $\tt P^*$ such that the relationship $\tt \Pr(Y|X_A)$  is reversed on $\tt P^*$, and thus $\tt h_{a,P^*}^{OPT}$ subtracts the same value.  Further, the input data distribution $\tt P^*(X)$ can be chosen such that $\tt \gamma_2 \geq \gamma_1$. That is, for a loss $\tt L$,  we can choose $\tt \lambda$ such that $ \tt  \mathcal{L}_{P^*}(h_{a,P}^{OPT}, y; \lambda) - \mathcal{L}_{P^*}(h_{c,P^*}^{OPT}, y) \geq \mathcal{L}_P(h_{c,P}^{OPT}, y) - \mathcal{L}_P(h_{a,P}^{OPT}, y; \lambda) $.

Hence, there exists a $\tt P^*$ such that $\tt \gamma_2 \geq \gamma_1$, and thus,
\begin{equation} \label{eqn:ode-bound-prob}
\tt ODE\texttt{-}Bound_{P,P^*}(h_c, y; \delta) \leq  ODE\texttt{-}Bound_{P,P^*}(h_a, y; \delta)
\end{equation}

\paragraph{Worst case ODE-bound for causal model is lower.} 
Finally, we show that the for a fixed $\tt P$, the worst case $\tt ODE\textnormal{-}Bound$ also follows Eqn.~\ref{eqn:ode-bound-prob}. Looking at Eqns.~\ref{eqn:ode-prob-causal} and \ref{eqn:ode-prob-assoc}, $\tt ODE\textnormal{-}Bound$  will be highest for a $\tt P^*$ such that discrepancy between $\tt P$ and $\tt P^*$ is highest and $\tt \mathcal{L}_{P^*}(h_{P}^{OPT}, y)$ is highest. Below we show that discrepancy $\tt disc_L(P, P^*)$ increases as   $\tt \mathcal{L}_{P^*}(h_{P}^{OPT}, y)$ increases. 
\footnotesize
\begin{align}
\tt    \mathcal{L}_{P^*}(h_{P}^{OPT}, y) &= \tt \mathcal{L}_{P^*}(h_{P}^{OPT}, y) - \mathcal{L}_{P}(h_{P}^{OPT}, y) + \mathcal{L}_{P}(h_{P}^{OPT}, y) \nonumber \\
                                   & \tt \leq disc_{L}(P, P^*) + \mathcal{L}_{P}(h_{P}^{OPT}, y) \nonumber \\
    \tt \Rightarrow disc_{L}(P, P^*) & \tt \geq    \mathcal{L}_{P^*}(h_{P}^{OPT}, y) -    \mathcal{L}_{P}(h_{P}^{OPT}, y)  
\end{align}
\normalsize
where $\tt \mathcal{L}_{P}(h_{P}^{OPT}, y)$ is fixed since $\tt P$ is fixed. Thus, the above equation shows that whenever $\tt \mathcal{L}_{P^*}(h_{P}^{OPT}, y)$ is high, discrepancy is also high. Hence, for any $\tt P^*_{max}$ that maximizes $\tt ODE\textnormal{-}Bound$, 
$\tt P^*_{max} = \arg \max_{P^*} ODE\textnormal{-}Bound_{P, P^*}(h, y; \delta )$, $\tt \mathcal{L}_{P^*}(h_{P}^{OPT}, y)$ is also maximized.  

Now, let us consider causal and associational models, and their respective worst case $\tt P^*_{max}$. 
To complete the proof, we need to check  whether $\tt \gamma_2\geq \gamma_1$ for such  maximal $\tt \mathcal{L}_{P^*}(h_{c,P}^{OPT}, y)$ and $\tt \mathcal{L}_{P^*}(h_{a,P}^{OPT}, y)$. Since $\tt \gamma_2$ increases monotonically with $\tt \mathcal{L}_{P^*}(h_{a,P}^{OPT}, y)$ ( $\tt \mathcal{L}_{P^*}(h_{c,P}^{OPT}, y)$ is bounded by $\tt \max_x L_x(h_{c,P}^{OPT}, y)$),  and there exists at least one $\tt P^*$ such that $\tt \gamma_2 \geq \gamma_1$,   this implies that $\tt \gamma_2 \geq \gamma_1$ for $\tt P^*_{max}$ too. Therefore, using Equation~\ref{eqn:ode-bound-prob},
\footnotesize
\begin{equation}
    \tt \max_{P^*} ODE\textnormal{-}Bound_{P, P^*}(h_c, y; \delta ) \leq \max_{P^*} ODE\textnormal{-}Bound_{P, P^*}(h_a, y; \delta ) 
\end{equation}
\normalsize
\end{proof}

\subsection{Generalization over a Single Datapoint}
\label{app:gen_cor}

\generalitycorollary*
\begin{proof}
    For any model $h$, we can write,
    \begin{equation}
        \tt   \max_{x \in S, x'}  L_{x'}(h, y)  - L_{x}(h, y) = \max_{x'}  L_{x'}(h, y)  - \min_{x \in S}L_{x}(h, y)  
    \end{equation}
    since $\tt  x'$ and $\tt x$ are independently selected. To prove the main result, we will show that the maximum loss on an unseen $x'$, $\max_{x'}L_{x'}(h, y)$ is higher for a loss-minimizing associational model than a causal model, and that minimum loss on a training point $x\in S$, $\tt \min_{x \in S}L_{x}(h, y)$ is lower for the associational model than a causal model.

\paragraph{Loss on a training data point.}First, consider loss on $\tt x\in S$,  $\tt L_x(h,y)$. 
    \begin{equation*}
    \begin{split}
        \tt h_{c,S}^{min}  \tt = \arg \min_{h\in \mathcal{H_C}} \mathcal L_S(h_c, y)  =&\tt  \arg \min_h \frac {1}{N} \sum_{i=1}^N L_{x_i}(h, y) \\
        \tt h_{a,S}^{min}  \tt = \arg \min_{h\in \mathcal{H_A}} \mathcal L_S(h_a, y)  =&\tt  \arg \min_h \frac {1}{N} \sum_{i=1}^N L_{x_i}(h, y) 
        \end{split}
    \end{equation*}
    Since $\tt \mathcal{H_C} \subseteq \mathcal{H_A}$, the average training loss will be lower for the associational model.
    \begin{equation}\label{eq:compare-avg-training-loss}
        \tt \mathcal L_S(h_{c,S}^{min}, y) \geq \mathcal  L_S(h_{a,S}^{min}, y)
    \end{equation}
    Further, under a suitably complex $\tt \mathcal H_A$ there exists a $\tt h_{a,S}^{min}$ such that the loss $\tt L$ is lower for any $\tt x \in S$. Therefore,
    \begin{equation} \label{eq:single-loss-min}
       \tt  \min_{x\in S} L_x(h_{c,S}^{min}, y) \geq\ \min_{x\in S} L_x(h_{a,S}^{min}, y) 
    \end{equation}

    \paragraph{Loss on an unseen data point.}
    Second, consider $\tt L_{x'}(h,y)$. Without loss of generality, let us write the true function for some $\tt (x',y') \sim P^*$ as,
    \begin{equation}
   \tt     y' =  h_{c,P^*}^{OPT}(x'_c) + \epsilon = h_{c,P}^{OPT}(x'_c) +  + \epsilon
    \end{equation}
    where we use that $\tt h_{c,P}^{OPT}=h_{c,P^*}^{OPT}$.  
    Suppose there is a data point $\tt (x'_1, y'_1)$ such that the loss $\tt L$ is maximum for $\tt h_{c,S}^{min}$.
\begin{equation} \label{eq:single-loss-hc}
    \begin{split}
        \tt     \max_{x' \not \in S} L_{x'}(h_{c,S}^{min},y) &\tt = L_{x'_1}(h_{c,S}^{min}(x'_1), y'_1) \\
                                                             &\tt = L_{x'_1}(h_{c,S}^{min}(x'_{c,1}), h_{c,P}^{OPT}(x'_{c,1})   + \epsilon_1)
    \end{split}
\end{equation}

Now for the associational model $\tt h_{a,S}^{min}$, the corresponding loss on $\tt x'_1$ is, 
\begin{equation} \label{eq:single-loss-ha}
    \begin{split}
        \tt    L_{x'_1}(h_{a,S}^{min},y) &\tt = L_{x'_1}(h_{a,S}^{min}, h_{c,P}^{OPT} +   \epsilon_1) 
\end{split}
\end{equation}

Without loss of generality, we can write the output of the associational model $\tt h_{a,S}^{min}$ on a particular input $\tt x'$ as,
    \begin{equation}
        \tt h_{a,S}^{min}(x') = h_{c,S}^{min}(x'_c) + h_a(x')
    \end{equation}
    where $\tt h_a$ is some associational function of $\tt x$. Therefore the loss on $\tt x'_1$ becomes,
    \begin{equation}
         \tt    L_{x'_1}(h_{a,S}^{min},y)    = L_{x'_1}(h_{c,S}^{min}+ h_{a}, h_{c,P}^{OPT} +  \epsilon_1)  
    \end{equation}
Since $\tt \Pr(Y|X_A)$ can change for different $\tt x' \sim P^*$ (where $\tt X_A=X\setminus X_C$ refers to the associational features), we will show that RHS of Eqn.~\ref{eq:single-loss-ha} can always be greater than or equal to the RHS of Eqn.~\ref{eq:single-loss-hc}.
For ease of exposition, we consider L1 loss below. For a causal model, the loss can be written as,
\begin{equation} \label{eq:single-loss-hc-l1}
    \begin{split}
        \tt     L_{x'_1}(h_{c,S}^{min}, &\tt h_{c,P}^{OPT} +  + \epsilon) \\
                                        &\tt = |h_{c,S}^{min}(x'_{c,1}) - h_{c,P}^{OPT}(x'_{c,1}) -   \epsilon_1 | \\
                                                                    &\tt = |h_{c,S}^{min}(x'_{c,1}) - h_{c,P}^{OPT}(x'_{c,1})| + | \epsilon_1 | 
\end{split}
\end{equation}
where $\tt x'_1$ (and thus $\tt  \epsilon_1$) is chosen such that $\tt \epsilon_1(h_{c,P}^{OPT}(x'_{c,1}) - h_{c,S}^{min}(x'_{c,1})) \geq 0$ which leads to maximum loss. 
And for the associational model, the loss on the same $\tt (x'_1, y'_1)$ can be written as,
\begin{equation}\label{eq:single-loss-ha-l1}
    \begin{split}
        \tt     & \tt L_{x'_1}(h_{a,S}^{min}, h_{c,P}^{OPT} +   \epsilon_1) \\
                & \tt= |h_{c,S}^{min}(x'_{c,1}) + h_a(x'_1) - h_{c,P}^{OPT}(x'_{c,1}) -   \epsilon_1 | \\
                & \tt = |(h_{c,S}^{min}(x'_{c,1}) - h_{c,P}^{OPT}(x'_{c,1}) ) +  (h_a(x'_1)  - \epsilon_1) | \\
\end{split}
\end{equation}
Comparing Eqns.~\ref{eq:single-loss-hc-l1} and \ref{eq:single-loss-ha-l1}, two cases arise. If $\tt h_a(x'_1)\epsilon_1 \leq 0 $, then we obtain,
\begin{equation} \label{eq:single-loss-ha-l1-x1}
    \begin{split}
        \tt     & \tt L_{x'_1}(h_{a,S}^{min}, h_{c,P}^{OPT} +   \epsilon_1) \\
                & \tt = |h_{c,S}^{min}(x'_{c,1}) - h_{c,P}^{OPT}(x'_{c,1}) |  +  |(h_a(x'_1)  - \epsilon_1) | \\
                & \tt = |h_{c,S}^{min}(x'_{c,1}) - h_{c,P}^{OPT}(x'_{c,1}) |  +  |h_a(x'_1)| + |  \epsilon_1 | 
\end{split}
\end{equation}
which is greater than maximum loss on $\tt x'_1$ using a causal model (Eqn.~\ref{eq:single-loss-hc-l1}). Otherwise, we can sample a new data point $\tt (x'_2, y'_2)$ from some other $ \tt P^*$ such that its causal features are the same ($\tt x'_{c,1}=x'_{c,2}$) and thus $\tt y$ is the same ($\tt y'_2 = y'_1 = h_{c,P}^{OPT}(x'_{c,1}) + \epsilon_1$), but its associational features are different ($\tt x'_{a,1}\neq x'_{a,2}$). Specifically, $\tt x'_{a,2}$ is chosen such that $\tt h_a(x'_2)\epsilon_1\leq 0$. Thus we again obtain, 
\begin{equation} \label{eq:single-loss-ha-l1-x2}
    \begin{split}
        \tt L_{x'_2}&\tt (h_{a,S}^{min}, h_{c,P}^{OPT} +  \epsilon_1) \\
                & \tt = |(h_{c,S}^{min}(x'_{c,2}) - h_{c,P}^{OPT}(x'_{c,2}) ) + (h_a(x'_2)  - \epsilon_1) | \\
                & \tt = |(h_{c,S}^{min}(x'_{c,1}) - h_{c,P}^{OPT}(x'_{c,1}) )  +  (h_a(x'_2) - \epsilon_1) | \\
                & \tt = |(h_{c,S}^{min}(x'_{c,1}) - h_{c,P}^{OPT}(x'_{c,1}) )|  +  |(h_a(x'_2)| + |  \epsilon_1 | 
\end{split}
\end{equation}
where the second equality uses   $\tt x'_{c,2}= x'_{c,1}$. Combining Eqns.~\ref{eq:single-loss-ha-l1-x1} and \ref{eq:single-loss-ha-l1-x2} and comparing to Eqn.~\ref{eq:single-loss-hc-l1}, we obtain, 
\begin{equation}\label{eq:single-loss-max}
    \tt   \max_{x'} L_{x'}(h_{c,S}^{min}, y) \leq\ \max_{x'} L_{x'}(h_{a,S}^{min}, y) 
\end{equation}

Finally, using Eqns.~\ref{eq:single-loss-min} and \ref{eq:single-loss-max} leads to the main result. 
\footnotesize
\begin{equation}
    \begin{split}
        & \tt   \max_{x'} L_{x'}(h_{c,S}^{min}, y) -  \min_{x\in S} L_{x}(h_{c,S}^{min}, y) \\ & \leq \max_{x'} L_{x'}(h_{a,S}^{min}, y) 
        - \min_{x\in S} L_{x}(h_{a,S}^{min}, y) \\
        & \tt   \max_{x', x\in S} L_{x'}(h_{c,S}^{min}, y) -  L_{x}(h_{c,S}^{min}, y) \leq \max_{x', x\in S} L_{x'}(h_{a,S}^{min}, y) -  
        L_{x}(h_{a,S}^{min}, y) \\
    \end{split}
\end{equation}
\normalsize

Using Eqns.~\ref{eq:compare-avg-training-loss} and \ref{eq:single-loss-max} we also obtain an auxiliary result.
\footnotesize
\begin{equation} \label{eq:single-input-max-avg}
    \tt   \max_{x'} L_{x'}(h_{c,S}^{min}, y) -  \mathcal L_S(h_{c,S}^{min}, y) \leq \max_{x'} L_{x'}(h_{a,S}^{min}, y) - \mathcal  L_S(h_{a,S}^{min}, y)
\end{equation}
\normalsize
\end{proof}

\section{Sensitivity of Causal and Associational Models}
\label{app:sensitivity}
Before we prove Lemma~\ref{lem:theta-sensitivity} on sensitivity, we prove Corollary~\ref{cor:gen-neighbor-datasets} and restate a Lemma from \cite{wu2015revisitingdp} for completeness.

\generalizeneighboringdatasets*
\begin{proof}
    Let $\tt S_{n-1}= S\setminus (x,y))$ and similarly $\tt S'_{n-1} = S'\setminus (x',y')$. Since $\tt S$ and $\tt S'$ differ in only one data point, $S_{n-1}=S'_{n-1}$. We will  add and subtract sum of losses on data points in  $\tt S_{n-1}$, $\mathcal (n-1) L_{S_{n-1}}$ to Theorem~\ref{cor:single-input} statement.

    Considering the LHS of Theorem~\ref{cor:single-input},
    \begin{equation} \label{eq:lossdiff-lhs}
        \begin{split}
            &\tt   \max_{x \in S, x'}  L_{x'}(h_{c,S}^{min}, y)  - L_{x}(h_{c,S}^{min}, y) \\
            &\tt =   \max_{x \in S, x'}  L_{x'}(h_{c,S}^{min}, y) + (n-1)\mathcal{L}_{S_{n-1}}(h_{c,S}^{min},y)  \\
            &\tt - L_{x}(h_{c,S}^{min}, y)  -  (n-1)\mathcal{L}_{S_{n-1}}(h_{c,S}^{min},y) \\
            &\tt =     \max_{x \in S, x'}  L_{x'}(h_{c,S}^{min}, y) + (n-1)\mathcal{L}_{S'_{n-1}}(h_{c,S}^{min},y)  \\
            &\tt - L_{x}(h_{c,S}^{min}, y)  -  (n-1)\mathcal{L}_{S_{n-1}}(h_{c,S}^{min},y) \\
            &\tt  =  \max_{S'} n\mathcal{L}_{S'}(h_{c,S}^{min},y)  -  n\mathcal{L}_{S}(h_{c,S}^{min},y) 
        \end{split}
    \end{equation}
    Similarly, the RHS of Theorem~\ref{cor:single-input} can be written as,
    \begin{equation} \label{eq:lossdiff-rhs}
        \begin{split}
            &\tt   \max_{x \in S, x'}  L_{x'}(h_{a,S}^{min}, y)  - L_{x}(h_{a,S}^{min}, y) \\
            &\tt =   \max_{x \in S, x'}  L_{x'}(h_{a,S}^{min}, y) + (n-1)\mathcal{L}_{S_{n-1}}(h_{a,S}^{min},y)  \\
            &\tt - L_{x}(h_{a,S}^{min}, y)  -  (n-1)\mathcal{L}_{S_{n-1}}(h_{a,S}^{min},y) \\
            &\tt =   \max_{S'} n\mathcal{L}_{S'}(h_{a,S}^{min},y)  -  n\mathcal{L}_{S}(h_{a,S}^{min},y) 
        \end{split}
    \end{equation}
    Using Theorem~\ref{cor:single-input} and dividing Eqns.~\ref{eq:lossdiff-lhs} and \ref{eq:lossdiff-rhs} by $\tt n$, we obtain,
    \begin{equation}
        \begin{split}
            &\tt   \max_{S'} n\mathcal{L}_{S'}(h_{c,S}^{min},y)  -  n\mathcal{L}_{S}(h_{c,S}^{min},y) \\
            &\tt \leq \max_{S'} n\mathcal{L}_{S'}(h_{a,S}^{min},y)  -  n\mathcal{L}_{S}(h_{a,S}^{min},y) \\
        \Rightarrow &\tt   \max_{S'} \mathcal{L}_{S'}(h_{c,S}^{min},y)  -  \mathcal{L}_{S}(h_{c,S}^{min},y) \\
            &\tt \leq \max_{S'} \mathcal{L}_{S'}(h_{a,S}^{min},y)  -  \mathcal{L}_{S}(h_{a,S}^{min},y) \\
    \end{split}
    \end{equation}
    Finally, since the above holds for any $\tt S\sim P$, it will also hold for the worst-case $\tt S$. The result follows. 
    \begin{equation}
        \begin{split}
         &\tt   \max_{S, S'} \mathcal{L}_{S'}(h_{c,S}^{min},y)  -  \mathcal{L}_{S}(h_{c,S}^{min},y) \\
            &\tt \leq \max_{S, S'} \mathcal{L}_{S'}(h_{a,S}^{min},y)  -  \mathcal{L}_{S}(h_{a,S}^{min},y) \\
    \end{split}
    \end{equation}
\end{proof}

\begin{lemma} \label{lem:wu2015-exchanging}
    [From \citet{wu2015revisitingdp}] Let $\tt S$ and $\tt S'$ be two neighboring datasets as defined in Corollary~\ref{cor:gen-neighbor-datasets} where $\tt S'=S\setminus (x,y) + (x',y')$. Given a model class $\tt \mathcal{H}$, Let $\tt h_{S}^{min}$ be the loss-minimizing model on $\tt S$ and $\tt  h_{S'}^{min}$ be the loss-minimizing model on $\tt S'$. Then the difference in losses between the two models on the same dataset is bounded by,
    \begin{equation}
        \begin{split}
        & \tt  \mathcal{L}_S(h_{S'}^{min}, y)-  \mathcal{L}_S(h_{S}^{min}, y)  \\
                                             & \tt \leq \frac{L_{x'}(h_{S}^{min}, y)- L_{x'}(h_{S'}^{min}, y)}{n} + \frac{L_{x}(h_{S'}^{min}, y) - L_{x}(h_{S}^{min}, y)}{n}
    \end{split}
    \end{equation}
\end{lemma}
\begin{proof}
    The proof follows from expanding loss over a dataset into individual  terms for each data point and then using the fact that $\tt h_{S'}^{min}$ has the minimum loss on $\tt S'$.
    
Using the definition of $\tt \mathcal{L}_S=\frac {1}{n} \sum_{i=1}^n L_{x_i}(h, y) $,  we can write the following for any two neighboring datasets $\tt S$ and $\tt S'$.
\begin{equation} \label{eq:sens-single-rho-ineq}
        \begin{split}
        \tt  \mathcal{L}_S &\tt (h_{S'}^{min}, y)-  \mathcal{L}_S(h_{S}^{min}, y)  \\
                           &\tt = \mathcal{L}_{S'}  (h_{S'}^{min}, y) + \frac{L_{x}(h_{S}^{min},y) - L_{x'}(h_{S}^{min},y)}{n} \\
                           &\tt - (\mathcal{L}_{S'}(h_{S}^{min}, y) + \frac{L_{x}(h_{S}^{min},y) - L_{x'}(h_{S}^{min},y)}{n})\\
                           &\tt = (\mathcal{L}_{S'}  (h_{S'}^{min}, y) - \mathcal{L}_{S'}(h_{S}^{min}, y))+ \frac{L_{x}(h_{S}^{min},y) - L_{x'}(h_{S}^{min},y)}{n} \\
                           &\tt + \frac{L_{x'}(h_{S}^{min},y) - L_{x}(h_{S}^{min},y) }{n})\\
                                             & \tt \leq \frac{L_{x'}(h_{S}^{min}, y)- L_{x'}(h_{S'}^{min}, y)}{n} + \frac{L_{x}(h_{S'}^{min}, y) - L_{x}(h_{S}^{min}, y)}{n}\\
    \end{split}
    \end{equation}
    where the last  inequality is since $\tt h_{S'}^{min}$ is the minimizer of $\tt L_{S'}(h,y)$ and thus  $\tt \mathcal{L}_{S'}  (h_{S'}^{min}, y) - \mathcal{L}_{S'}(h_{S}^{min}, y) \leq 0$.

\end{proof}

\sensitivitylemma*
\begin{proof}
Since $\tt  L$ is a strongly convex function over $\Omega$, we can write for the two models $\tt h_{c,S}^{min}$ and $\tt h_{c,S'}^{min}$ trained on $\tt S$ and $\tt S'$ respectively~\cite{wu2015revisitingdp},
    \begin{equation} \label{eq:sens-strong-convex}
        \begin{split}
        & \tt        \mathcal{L}_S(h_{c,S}^{min}, y) \leq \mathcal L_S(\alpha h_{c,S}^{min} + (1-\alpha)h_{c,S'}^{min}, y) \\
                                &\tt  \leq \alpha \mathcal L_S (h_{c,S}^{min}, y) + (1-\alpha) \mathcal L_S(h_{c,S'}^{min}, y) \\
                                &\tt -  \frac{\lambda}{2}\alpha (1-\alpha) ||h_{c,S'}^{min} - h_{c,S}^{min}||^2 
    \end{split}
    \end{equation}
    where $\tt \alpha \in (0,1)$ and  the first inequality is since $\tt h_{c,S}^{min}$ is the loss-minimizing model over $\tt S$. Rearranging the terms and tending $\alpha$ to $1$ leads to,
    \begin{equation}\label{eq:sens-strong-convex2}
        \begin{split}
            &\tt (1-\alpha) (\mathcal L_{S}( h_{c,S}^{min}, y) -  \mathcal L_S(h_{c,S'}^{min}, y)) \\
            &\tt \leq -  \frac{\lambda}{2} \alpha(1-\alpha) \norm{h_{c,S'}^{min} - h_{c,S}^{min}}^2 \\ 
            \tt \Rightarrow &\tt  \frac{\lambda}{2} \norm{h_{c,S'}^{min} - h_{c,S}^{min}}^2 \leq \mathcal L_{S}( h_{c,S'}^{min}, y) - \mathcal L_S (h_{c,S}^{min}, y)  
    \end{split}
    \end{equation}
    Now consider $\tt \max_{S,S'} \norm{ h_{c,S}^{min}- h_{c,S'}^{min} }_1$. Without loss of generality, we can order the pair of datasets $\tt S,S'$ such that $\tt \mathcal L_S( h_{c,S'}^{min}, y) \leq \mathcal L_{S'}(h_{c,S}^{min}, y)$. Then using Eqn.~\ref{eq:sens-strong-convex2} and taking the maximum, we obtain,
    \begin{equation} \label{eq:sens-strong-convex-max}
        \begin{split}
        \tt  \frac{\lambda}{2} \max_{S,S'} \norm{ h_{c,S}^{min}- h_{c,S'}^{min} }_1^2 \leq \max_{S,S'} \mathcal L_{S}( h_{c,S'}^{min}, y) - \mathcal L_S (h_{c,S}^{min}, y)  \\
        \tt \leq \max_{S,S'} \mathcal L_{S'}( h_{c,S}^{min}, y) - \mathcal L_S (h_{c,S}^{min}, y)  \\
        \tt \leq \max_{S,S'} \mathcal L_{S'}( h_{a,S}^{min}, y) - \mathcal L_S (h_{a,S}^{min}, y)  \\
    \end{split}
    \end{equation}
    where the last inequality is due to Theorem~\ref{cor:single-input}. Let $S_1$ and $S'_1$ denote the datasets that lead to the maximum in the RHS above. 
    We know that   $\tt \mathcal L_{S_1}( h_{a,S'_1}^{min}) \geq \mathcal L_{S'_1}( h_{a,S'_1}^{min})$ since $\tt h_{a,S'_1}^{min} $ is the loss-minimizing model over $S'_1$.  Therefore, we can rewrite,
    \begin{equation} \label{eq:sens-max-inst}
        \begin{split}
     &\tt   \max_{S,S'} \mathcal L_{S'}( h_{a,S}^{min}, y) - \mathcal L_S (h_{a,S}^{min}, y)  \\
     & \tt   = \mathcal L_{S'_1}( h_{a,S_1}^{min}, y) - \mathcal L_{S_1} (h_{a,S_1}^{min}, y)  \\
            &\tt  \leq \mathcal L_{S'_1}( h_{a,S_1}^{min}, y) - \mathcal L_{S_1} (h_{a,S_1}^{min}, y)  + (\mathcal L_{S_1}( h_{a,S'_1}^{min}) - \mathcal L_{S'_1}( h_{a,S'_1}^{min})) \\
            \tt  &\tt = [\mathcal L_{S'_1}( h_{a,S_1}^{min}, y) - \mathcal L_{S'_1}( h_{a,S'_1}^{min}) ]+  (\mathcal L_{S_1}( h_{a,S'_1}^{min}) - \mathcal L_{S_1} (h_{a,S_1}^{min}, y)]
\end{split}
    \end{equation}
    Now using Lemma~\ref{lem:wu2015-exchanging}, we obtain the following two bounds.
\begin{equation} \label{eq:sens-single-rho-ineq}
        \begin{split}
            &\tt  \mathcal{L}_{S_1}(h_{a,S'_1}^{min}, y)-  \mathcal{L}_{S_1}(h_{a,S_1}^{min}, y)  \\
                                             & \tt \leq \frac{L_{x'}(h_{a,S_1}^{min}, y)- L_{x'}(h_{a,S'_1}^{min}, y)}{n} + \frac{L_{x}(h_{a,S'_1}^{min}, y) - L_{x}(h_{a,S_1}^{min}, y)}{n}\\
                                             &\tt  \mathcal{L}_{S'_1}(h_{a,S_1}^{min}, y)-  \mathcal{L}_{S'_1}(h_{a,S'_1}^{min}, y)  \\
                                             & \tt \leq \frac{L_{x'}(h_{a,S_1}^{min}, y)- L_{x'}(h_{a,S'_1}^{min}, y)}{n} + \frac{L_{x}(h_{a,S'_1}^{min}, y) - L_{x}(h_{a,S_1}^{min}, y)}{n}\\
    \end{split}
    \end{equation}

    Now since the loss function $\tt L(.,y)$ is $\tt \rho$-Lipschitz, we have $\tt L_{x}(h_1, y)- L_{x}(h_2, y) \leq \rho \norm{h_1-h_2}_1$ for any data point $\tt x$ and any two models $\tt h_1$ and $\tt h_2$. Plugging Eqn.~\ref{eq:sens-single-rho-ineq} and the $\rho$-Lipschitz property back in Eqn.~\ref{eq:sens-max-inst}, 
\begin{equation}
    \begin{split}
     &   \tt \mathcal L_{S'_1}( h_{a,S_1}^{min}, y) - \mathcal L_{S'_1}( h_{a,S'_1}^{min}, y) \leq \frac{2}{n}\rho \norm{h_{a,S_1}^{min} - h_{a,S'_1}^{min}}_1\\
     &  \tt \mathcal L_{S_1}( h_{a,S'_1}^{min}, y) - \mathcal L_{S_1}( h_{a,S_1}^{min}, y) \leq \frac{2}{n}\rho \norm{h_{a,S_1}^{min} - h_{a,S'_1}^{min}}_1\\
      \tt   \Rightarrow &  \tt  \max_{S,S'} \mathcal L_{S'}( h_{a,S}^{min}, y) - \mathcal L_S (h_{a,S}^{min}, y)  \\
                        & \tt  \leq  [\mathcal L_{S'_1}( h_{a,S_1}^{min}, y) - \mathcal L_{S'_1}( h_{a,S'_1}^{min}) ]+  (\mathcal L_{S_1}( h_{a,S'_1}^{min}) - \mathcal L_{S_1} (h_{a,S_1}^{min}, y)]\\
                        & \tt \leq \frac{2}{n}\rho \norm{h_{a,S_1}^{min} - h_{a,S'_1}^{min}}_1 + \frac{2}{n}\rho \norm{h_{a,S_1}^{min} - h_{a,S'_1}^{min}}_1\\
    &\tt =\frac{4}{n}\rho \norm{h_{a,S_1}^{min} - h_{a,S'_1}^{min}}_1  
    \end{split}
\end{equation}
    Finally, combining with  Eqn.~\ref{eq:sens-strong-convex-max}, we obtain,
    \begin{equation}
        \begin{split}
            \tt \max_{S,S'}       \norm{h_{c,S'}^{min} - h_{c,S}^{min}}^2_1 &\tt \leq \frac{8\rho}{\lambda n} \norm{ h_{a,S_1}^{min} - h_{a,S'_1}^{min}}_1 \\
                                                                            &\tt \leq \frac{8\rho}{\lambda n} \max_{S,S'} \norm{ h_{a,S}^{min} - h_{a,S'}^{min}}_1 \\
                                                                            &\tt  \leq  \max_{S,S'} \norm{ h_{a,S}^{min} - h_{a,S'}^{min}}_1 
        \end{split}
    \end{equation}
    where the last inequality holds for a sufficiently large $\tt n$ such that  $\tt \frac{8\rho}{\lambda n} \leq 1 $. If $\tt \max_{S,S'}\norm{h_{c,S'}^{min} - h_{c,S}^{min}}_1 \geq 1$, the result follows. Otherwise, we need a larger $\tt n$ such that $\tt n \geq \frac{8\rho}{\lambda \max_{S,S'}\norm{h_{c,S'}^{min} - h_{c,S}^{min}}_1 }$. In both cases, we obtain, 
    \begin{equation}
        \tt   \max_{S, S'} \norm{ h_{c, S}^{min} - h_{c, S'}^{min}}_1  \leq \max_{S, S'}\norm{ h_{a, S}^{min} - h_{a, S'}^{min}}_1 
    \end{equation}

    Hence, sensitivity of a causal model is lower than an associational model, i.e., $\Delta \mathcal F_c \leq \Delta \mathcal F_a$. 
\end{proof}

\section{Differential Privacy Guarantees with Tighter Data-dependent Bounds}
\label{app:voting}

In this section we provide the differential privacy guarantee of a causal model based on a recent method~\citep{papernot2016teacher-dp} that provides tighter data-dependent bounds. 

As a consequence of Theorem~\ref{cor:single-input}, another generalization property of causal learning is that classification models trained on data from two different distributions $\tt P(X)$ and $\tt P^*(X)$ are likely to output the same value for a new input.
\begin{lemma}\label{cor:loss-compare}
    Under the conditions of Theorem~\ref{thm:causal-gen-bounds} and 0-1 loss, let $\tt h_{c,S}^{min}$ be the loss-minimizing causal classification model trained on a dataset $\tt S$ from distribution $\tt P$ and let $\tt h_{c,S^*}^{min}$ be the loss-minimizing model trained on a dataset $\tt S^*$ from $\tt P^*$. Similarly, let $\tt h_{a,S}^{min}$ and $\tt h_{a,S^*}^{min}$ be loss-minimizing associational classification  models trained on $\tt S$ and $\tt S^*$ respectively.  
    Then for any new data input $\tt x$,
    \begin{equation*}
        \begin{split}
            \tt \min_{S \sim P, S^*\sim P^*}&\tt \Pr(h_{c,S}^{min}(x)=h_{c,S^*}^{min}(x)) \\
                                            & \tt \geq  \min_{S \sim P, S^*\sim P^*} \Pr(h_{a,S}^{min}(x)=h_{a,S^*}^{min}(x))
\end{split}
\end{equation*}
    As the size of the training sample $\tt |S|=|S^*| \rightarrow \infty$, the LHS$ \rightarrow 1$.
\end{lemma}

\begin{proof}
    Let $\tt h_{a,P}^{min}=\arg \min_{h\in \mathcal{H}_A} \mathcal{L}_S(h, y)$ and  $\tt h_{a,P^*}^{min}=\arg \min_{h\in \mathcal{H}_A} \mathcal{L}_{S^*}(h, y)$ be the loss-minimizing associational hypotheses under the two datasets $\tt S$ and $\tt S^*$ respectively, where $\tt \mathcal{H}_A$ is the set of hypotheses. We can analogously define $\tt h_{c, P}^{min}$ and $\tt h_{c,P^*}^{min}$.  
Likewise, let $\tt h_{a,P}^{OPT}=\arg \min_{h\in \mathcal{H}_A} \mathcal{L}_P(h, y)$ and similarly let  $\tt h_{a,P^*}^{OPT}=\arg \min_{h\in \mathcal{H}_A} \mathcal{L}_{P^*}(h, y)$ be the loss-minimizing hypotheses over the two distributions.  We can analogously define $\tt h_{c, P}^{OPT}$ and $\tt h_{c,P^*}^{OPT}$. For ease of exposition, let us consider a binary classification task where all associational models map $\tt X \to \{0,1\}$ and causal models map $\tt X_C \to \{0,1\}$.

\paragraph{Infinite sample result.}
As $\tt |S|=|S^*| \rightarrow \infty$, each of models on $S$ and $S^*$ approach their loss-minimizing functions on the distributions $P$ and $P^*$ respectively.  Then, for any input $\tt x$,
\begin{align}
    \tt \lim_{|S|\rightarrow \infty} h_{a,S}^{min} &\tt = h_{a,P}^{OPT} &
    \tt \lim_{|S^*|\rightarrow \infty} h_{a,S^*}^{min} &\tt = h_{a,P^*}^{OPT}\\
    \tt \lim_{|S|\rightarrow \infty} h_{c,S}^{min} &\tt = h_{c,P}^{OPT} &
    \tt \lim_{|S^*|\rightarrow \infty} h_{c,S^*}^{min} &\tt = h_{c,P^*}^{OPT}
\end{align}
From Theorem~\ref{thm:causal-gen-bounds} (Equation~\ref{eq:equal_h}), we know that $\tt h_{c,P}^{OPT}= h_{c,P^*}^{OPT}$. Therefore, for any new input $\tt x$ for a causal model, we obtain $\tt \Pr(h_{c,P}^{OPT}(x)=h_{c,P^*}^{OPT}(x))=1 $, but not necessarily for associational models. This leads to, 
\begin{align}
    \tt \lim_{|S|\rightarrow \infty} & \tt \Pr(h_{c,S}^{min}(x)=h_{c,S^*}^{min}(x))  =1  \\
                                     &
                                     \tt \geq \lim_{|S^*|\rightarrow \infty} \Pr(h_{a,S}^{min}(x)=h_{a,S^*}^{min}(x))
\end{align}

\paragraph{Finite sample result.}
Under finite samples, let $\tt S_1$ and $\tt S^*_2$ be the two datasets from $\tt  P$ and $\tt P^*$ respectively that lead to the minimum probability of agreement between the two causal models $\tt h_{c,S_1}^{min}$ and $\tt h_{c,S^*_1}^{min}$.
\begin{equation}
    \tt \min_{S \sim P, S^*\sim P^*}\Pr(h_{c,S}^{min}(x)=h_{c,S^*}^{min}(x)) = \Pr(h_{c,S_1}^{min}(x)=h_{c,S^*_1}^{min}(x)) 
\end{equation}
Now consider the probability of agreement for the two associational models trained on the same datasets, $\tt h_{a,S_1}^{min}$ and $\tt h_{a,S^*_1}^{min}$. Without loss of generality, we can write the associational models as,
\begin{equation}
    \begin{split}
        \tt h_{a,S_1}^{min}(x) = |h_{c,S_1}^{min}(x) - h_{a,S_1}(x)|\\
    \tt h_{a,S^*_1}^{min}(x) = |h_{c,S^*_1}^{min}(x) - h_{a,S^*_1}(x)|
    \end{split}
\end{equation}
where $\tt h_{a,S_1}: X\to \{0,1\}$ and $\tt h_{a,S^*_1}: X\to \{0,1\}$ are any two functions. Effectively, when $\tt h_{a,S_1}$ is $\tt 1$, it flips the output of the loss-minimizing associational model  compared to the loss-minimizing causal model on $\tt S_1$ (and similarly for $\tt h_{a,S^*_1}$ on $\tt S^*_1$).

Now we can select a different  $\tt S^*_2 \sim P^*_2$ where $\tt y$ and the causal features remain the same as $\tt S^*_1$ but associational features are changed for each input $\tt x\in S^*_2$. Therefore $\tt h_{c,S^*_1}^{min}=h_{c,S^*_2}^{min}$  but the  loss-minimizing associational model $\tt h_{a,S^*_2}$ has the following property. $\tt h_{a,S^*_2}\neq  h_{a,S_1}(x)$ if $\tt h_{c,S^*_1}^{min} = h_{c,S_1}^{min}$, and  $\tt h_{a,S^*_2} =  h_{a,S_1}(x)$ if $\tt h_{c,S^*_1}^{min} \neq h_{c,S_1}^{min}$. Under $\tt S^*_2$,
\begin{equation}
    \begin{split}
\tt        \abs{h_{a,S_1}^{min} - h_{a,S^*_2}^{min}} \geq 
  \tt      \abs{h_{c,S_1}^{min} - h_{c,S^*_2}^{min}} \\ 
    \tt    = \abs{h_{c,S_1}^{min} - h_{c,S^*_1}^{min}} =
      \tt  \max_{S,S^*}\abs{h_{c,S}^{min} - h_{c,S^*}^{min}} 
\end{split}
\end{equation}
Therefore, the disagreement between two associational models trained on two datasets is greater than or equal to the disagreement between causal models on the worst-case $\tt S_1$ and $\tt S^*_1$. Since the loss is 0-1 loss,  the worst-case probability of agreement is lower. 
\begin{equation*}
    \begin{split}
      &\tt  \max_{S,S^*}\abs{h_{c,S}^{min}(x) - h_{c,S^*}^{min}(x)} \leq 
      \tt  \max_{S,S^*}\abs{h_{a,S}^{min}(x) - h_{a,S^*}^{min}(x)} \\
       \tt  \Rightarrow &\tt \min_{S\sim P,S^*\sim P^*}  \Pr(h_{c,S}^{min}(x)=h_{c,S^*}^{min}(x)) \\
                    &\tt                                       \geq  \min_{S\sim P,S^*\sim P^*} \Pr(h_{a,S}^{min}(x)=h_{a,S^*}^{min}(x))
                                  \end{split}
                                  \end{equation*}

\end{proof}

Based on the above generalization property, we now show that causal models provide stronger  differential privacy guarantees than corresponding associational models.
We utilize the subsample and aggregate technique \citep{dwork2014algorithmic} that was extended for machine learning in \citet{hamm2016learning-dp} and \citet{papernot2016teacher-dp}, for constructing a differentiably private model. The framework considers $\tt M$ arbitrary teacher models  that are trained on a separate subsample of the dataset without replacement. Then, a student model is trained on some auxiliary unlabeled data with the (pseudo) labels generated from a majority vote of the teachers. Differential privacy can be achieved by either perturbing the number of votes for each class~\citep{papernot2016teacher-dp}, or perturbing the learnt parameters of the student model~\citep{hamm2016learning-dp}.
For any new input, the output of the model is a majority vote on the predicted labels from the  $\tt M$ models. 
The privacy guarantees are better if a larger number of teacher models agree on each  input, since by definition the majority decision could not have been changed by modifying a single data point (or a single teacher's vote). 
Since causal models generalize to new distributions, intuitively we expect causal models trained on separate samples to agree more.  Below we show that for a fixed amount of noise, a causal model is $\epsilon_c$-DP compared to $\epsilon$-DP for a associational model, where $\epsilon_c \leq \epsilon$. 
\votingtheorem*

\begin{proof}
    Consider a change in a single input example $(x, y)$, leading to a new  $D'$ dataset. Since sub-datasets are sampled without replacement, only a single teacher model can change in $D'$.
    Let $n'_j$ be the vote counts for each class under $D'$. Because the change in a single input can only affect one model's vote, $ |n_j - n'_j| \leq 1$.

    Let the noise added to each class be $r_j \sim Lap(2/\gamma)$. Let the majority class (class with the highest votes) using data from $D$ be $i$ and the class with the second largest votes be $j$. 
    Let us consider the minimum noise $r^*$ required for class $i$ to be the majority output under $\mathcal{M}$ over $D$. Then,
    $$ n_i + r^* > n_j + r_j $$
    For $i$ to have the maximum votes using $\mathcal{M}$ over $D'$ too, we need,
    $$ n'_i+ r_i > n'_j + r_j $$
    In the worst case, $n'_i=n_i -1$ and $n'_j=n_j+1$ for some $j$. Thus, we need,
    \begin{align}
        n_i -1 + r_i > n_j +1+ r_j &\Rightarrow
 n_i + r_i > n_j +2+ r_j
\end{align}
which shows that $r_i >r*+2$. Note that $r*> r_j - (n_i-n_j)$. We have two cases:

\paragraph{CASE I:}The noise $r_j < n_i -n_j$, and therefore $r^*<0$.  
Writing $\Pr(i|D')$ to denote the probability that class $i$ is chosen as the majority class under $D'$,
\begin{equation}
\begin{split}
  \tt  P(i|D')  = P(r_i \geq r^*+2)  & \tt = 1 - 0.5\exp(\gamma)\exp(\frac{1}{2}\gamma r^*)\\
            & \tt = 1 - \exp(\gamma)(1- P(r_i \geq r^*)) \\
            & \tt = 1 - \exp(\gamma)(1- P(i|D))
\end{split}
\end{equation}
where the equations on the right are due to Laplace c.d.f.
Using the above equation, we can write:
\begin{equation}
\begin{split}
    \tt \frac{P(i|D')}{P(i|D)} & \tt = \exp( \gamma) + \frac{1 - \exp(\gamma)}{P(i|D)} \\
    &  \tt = \exp(\gamma) + \frac{1 - \exp(\gamma)}{P(r_i \geq r^*)} 
    \leq \exp(\epsilon) \label{eqn:eps-dp}
\end{split}
\end{equation}
for some $\epsilon>0$. As $P(i|D)= P(r_i \geq r^*)$ increases, the ratio decreases and thus the effective privacy budget ($\epsilon$) decreases. Thus, a DP-mechanism  based on teacher models with  lower $r^*$ (effectively higher $|r^*|$) will exhibit the lowest $\epsilon$.

Below we show that the worst-case $|r^*|$ across any two datasets $\tt S \sim P$, $\tt S^* \sim P^*$ such that $\tt P(Y|X_C)=P^*(Y|X_C)$  is higher for a causal model, and thus $\tt \max_D P(r_i \geq r^*)$ is higher. Intuitively, $|r^*|$ is higher when there is more consensus between the $\tt M$ teacher models since $|r^*|$ is the difference between the votes for the highest voted class with the votes for the second-highest class. 
For two sub-datasets $\tt D_1 \subset D$ and  $\tt D_2 \subset D$, let the two causal teacher models be $\tt h_{c,D_1}$ and $\tt h_{c,D_2}$, and the two associational  teacher models be $\tt h_{a,D_1}$ and $\tt h_{a,D_2}$. From Lemma~\ref{cor:loss-compare},  for any new $\tt x$,  there is more consensus among causal models.  
\begin{align*}
&    \tt \min_{D_1, D_2}\Pr(h_{c,D_1}(x)=h_{c,D_2}(x)) \geq \min_{D_1, D_2}\Pr(h_{a,D_1}(x)=h_{a,D_2}(x))\\
    & \tt \Rightarrow  \min_D \min_{D_1, D_2}\Pr(h_{c,D_1}(x)=h_{c,D_2}(x)) \\
    &\tt \geq \min_D \min_{D_1, D_2}\Pr(h_{a,D_1}(x)=h_{a,D_2}(x))
\end{align*}
Hence worst-case $r^*_c \leq r^*$. From Equation~\ref{eqn:eps-dp}, $\epsilon_c \leq \epsilon$. 

\paragraph{CASE II: }The noise $r_j >= n_i -n_j$, and therefore $r^*>=0$. Following the steps above, we obtain: 
\begin{equation}
\begin{split}
   \tt  P(i|D') = P(r_i \geq r^*+2) & \tt = 0.5\exp(- \gamma)\exp(- \frac{1}{2}\gamma r^*)\\
            & \tt = \exp(- \gamma)(P(r_i \geq r^*)) \\ & \tt= \exp(- \gamma)(P(i|D))
\end{split}
\end{equation}
Thus, the ratio does not depend on $r^*$.   
\begin{align}
    \frac{P(i|D')}{P(i|D)} &= \exp(- \gamma)
\end{align}
Under CASE II when the noise is higher to the differences in votes between the highest and second highest voted class, causal models provide the same privacy budget as associational models.  

%
Thus, overall, $\epsilon_c \leq \epsilon$. 
\end{proof}

\section{Maximum Advantage of a Differentially Private algorithm}

\maxadvdp*
\begin{proof}
Consider a neighboring dataset $S'$ to $\tt S\sim P$ such that $\tt S'$ replaces data point  $x\in S$ with a different point $x'$.  Let $F_S$ and $F_S'$ be differentially private mechanisms trained on $S$ and $S'$ respectively. Following Theorem 1 proof from \cite{yeom2018privacy}, the membership advantage of an adversary $\mathcal{A}$ on a differentially private algorithm $\hat{\mathcal{F}}$ can be written as:
    \begin{equation}
        \begin{split}
            \tt Adv& \tt (\mathcal{A}, \hat{\mathcal{F}}, n, P, P^*)  =\Pr(\mathcal{A}=1|b=1) -\Pr(\mathcal{A}=1|b=0 )  \\
             & \tt = \Pr(\mathcal A(x, F_S) = 1 |x\in S) - \Pr(\mathcal A(x, F_{S'}) = 1|x \in S)
             \end{split}
         \end{equation}
         where $\mathcal A(x, F_S)$ denotes a membership adversary for an algorithm $F_S$ trained on a dataset $S$, and $\tt \mathcal A(x, F_{S'})$ denotes an adversary attacking algorithm $F_{S'}$ trained on $S'$. Without loss of generality for the case where there are an infinite number of  models $h$, assume that models are sampled from a discrete set of K models: $\{ h_1, h_2, ..., h_K\}$. Then using the law of total probability over the models yielded by the algorithms $F_S$ and $F_{S'}$, 
        \begin{equation}
            \begin{split}
        \tt Adv& \tt (\mathcal{A}, \hat{\mathcal{F}}, n, P, P^*)  
        \tt = \sum_{j=1}^{K} \Pr(\mathcal A(x, h_j)=1)\Pr(F_S=h_j)  \\
               &\tt - \sum_{j=1}^{K}  \Pr(\mathcal A(x,h_j)=1)\Pr(F_{S'}=h_j) \\
               &\tt = \sum_{j=1}^{K} \Pr(\mathcal A(x, h_j)=1) [ \Pr(F_S=h_j)  - \Pr(F_{S'}=h_j) ]
        \end{split}
    \end{equation}
    
    where $\Pr(\mathcal A(x,h_j))$ can be interpreted as the weights in a sum. 
    Thus, the above is a weighted sum and will be maximum when positive values for $ \Pr(F_S=h_j)  - \Pr(F_{S'}=h_j)$ have the highest weight and negative values for $   \Pr(F_S=h_j)  - \Pr(F_{S'}=h_j)$ have  zero weight. It follows that to obtain the maximum advantage,  the adversary will  choose $\Pr(\mathcal A(x, h_j)=1)=1$ if $ \Pr(F_S=h_j)  - \Pr(F_{S'}=h_j) >0$, and $0$ otherwise.  In other words, the adversary predicts membership in train set for an input $\tt x$ whenever probability of the given model $h_j$ being generated from $\tt F_S$ is higher than it being generated from $F_{S'}$. 

Let $H_+ \subset H$ be the set of models for which $\Pr(F_S=h_j)  - \Pr(F_{S'}=h_j)>0$. Similarly,  let $H_- = H \setminus H_+$ be the set of models for which being generated from $F_{S'}$ is more probable: $\Pr(F_{S'}=h_j)  - \Pr(F_{S}=h_j)\geq 0$.   The worst-case adversary selects datasets $\tt S \sim P,S'$ such that the sum  $ \sum_{h_j \in H_+}\Pr(F_S=h_j)  - \Pr(F_{S'}=h_j) $ is the highest. Therefore, for a given distribution $P$ and a differentially private algorithm $F_S$ learnt on $S\sim P$, we can write the maximum membership advantage as,
    \begin{equation} \label{eqn:max-adv-dp}
        \begin{split}
            \tt \max_{\mathcal A, P^*} &\tt Adv(\mathcal A, F_S, n, P, P^*)\\
                     &\tt = max_{S,S'} \sum_{h_j \in H_+} [ P(F_S=h_j)  - P(F_S'=h_j) ]   \\
             &\tt = max_{S,S'} P(F_S \in H_+) - P(F_S' \in H_+)\\
             & \tt     = max_{S,S'}  P(F_S \in H_+) - (1- P(F_S' \in H_-))     \\
             &\tt =  max_{S,S'} 2\Pr(F_S \in H_+) -1 
\end{split}
\end{equation}
where the last equality is since the Laplace noise is  added to $F_{S}$ and $F_{S'}$ is from identical distributions and thus $\Pr(F_S \in H_+)= \Pr(F_{S'} \in H_-) $. 
Equation~\ref{eqn:max-adv-dp} provides the maximum membership advantage for any $\epsilon$-DP mechanism $F_S$ with Laplace noise.

We next show that Eqn.~\ref{eqn:max-adv-dp} for a  causal mechanism $F_c$ is not greater than that for an associational mechanism $F_a$. 
Let $\Pr(F_S)$ be a Laplace distribution with mean at $h_{S,min}$ and $\Pr(F_{S'})$ be a Laplace distribution with mean $h_{S',min}$ with identical scale/noise parameter. We would like to find the  boundary model $h^{\dagger}$ of the set $H_+$ where $ P(F_S=h_j)  = P(F_S'=h_j)$, since $\Pr(F_S \in H_+)$ is the probability under the Laplace distribution cut off at a point $h^\dagger$. Due to identical noise for $F_S$ and $F_S'$ and the symmetry of the Laplace distribution, the boundary $h^\dagger$ corresponds to the midpoint of $h_{S,min}$ and $h_{S',min}$: $\tt 0.5(h_{S,min}+h_{S',min})$. Alternatively, the L1 distance of the boundary $h^{\dagger}$ from the means of the Laplace distributions  can be written as (for a worst $S,S'$),
\begin{equation}
\tt     \norm{h^{\dagger} - h_{S,min}}_1 = \frac{\norm{h_{S',min}-h_{S,min}}_1}{2} = \frac{\Delta  F_S }{2}
\end{equation}
where $\Delta F_S$ is the sensitivity of $F_S$ and the last equality is due to the choice of worst-case $S$ and $S'$. 

    From Lemma~\ref{lem:theta-sensitivity}, we know that sensitivity of a causal learning function is lower than that of an associational learning model. 
    \begin{equation}
       \tt  \Delta \hat{\mathcal F}_{c,S} \leq \Delta \hat{\mathcal F}_{a,S}
    \end{equation}
    Thus, L1 distance of $h_c^\dagger$ from the mean $h_{c,S}^{min}$ is lower for a causal learning function, and thus its PDF $\Pr(F_S=h_j)$ is higher. Now the set H+ is a one-sided boundary on the values of $h$ and includes the mean of the Laplace distribution. Given symmetry of the  Laplace distribution, probability of $F_S$ lying in $H_+$, $\tt \Pr(F_S \in H_+)$ should be lower whenever the PDF at the one-sided boundary is higher. Therefore, $P(F_S in H+)$ is lower for a causal mechanism than the associational learning mechanism. 
    \begin{equation}
       \tt  \Delta \hat{\mathcal F}_{c,S} \leq \Delta \hat{\mathcal F}_{a,S}
       \Rightarrow
       \Pr(\hat{\mathcal F}_{c,S} \in H_+) \leq \Pr(\hat{\mathcal F}_{a,S} \in H_+)
    \end{equation}
    Finally, using the above equation in Eqn.~\ref{eqn:max-adv-dp} shows that the maximum membership advantage of a causal model is lower. 
    \begin{equation}
        \tt \max_{\mathcal A, P^*}  Adv(\mathcal A, \hat{\mathcal F}_{c,S}, n, P, P^*) \leq 
        \tt \max_{\mathcal A, P^*} Adv(\mathcal A, \hat{\mathcal F}_{a,S}, n, P, P^*)
    \end{equation}
\end{proof}

\section{Infinite sample robustness to membership inference attacks} \label{appsec:infty-attack}

\inftyattack*
\begin{proof}
    $\tt h_{c,S}^{min}$ can be obtained by empirical risk minimization.
\begin{equation}
    \tt    h_{c,S}^{min} = \arg \min_{h\in \mathcal{H}_C}  \mathcal L_{S\sim P}(h, y)  = \arg \min_{h\in \mathcal{H}_C}  \frac {1}{n} \sum_{i=1}^n L_{x_i}(h,y) 
\end{equation}
As $\tt |S| = n \rightarrow \infty$, $\tt h_{c,S}^{min} \rightarrow h_{c,P}^{OPT}$. Suppose now that there exists another $\tt S'$ of the same size such that $\tt S'\sim P^*$. Then  as $\tt |S'|\rightarrow \infty$, $\tt h_{c,S'}^{min} \rightarrow h_{c,P^*}^{OPT}$.

From Theorem~\ref{thm:causal-gen-bounds}, $\tt h_{c,P}^{OPT} = h_{c,P^*}^{OPT}$. Thus, 
\begin{equation} \label{eqn:infty-equal-h}
 \tt   \lim_{n \rightarrow \infty} h_{c,S}^{min} =\lim_{n \rightarrow \infty} h_{c,S'}^{min}
\end{equation}
Equation~\ref{eqn:infty-equal-h} implies that as $\tt n\rightarrow \infty$, the learnt $\tt h_{c,S}^{min}$ does not depend on the training set, as long as the training set is sampled from any distribution $\tt P^*$ such that $\tt P(Y|X_C)=P^*(Y|X_C)$.
That is, being the global minimizer over  distributions, $\tt h_{c,S}^{min}= h_{c,P}^{OPT}$ does not depend on its training set. Therefore, $\tt h_{c,S}^{min}(x)$ is independent of  whether $\tt x$ is in the training set. 
\begin{align}
    \begin{split}
\tt    \lim_{n \rightarrow \infty} Adv(\mathcal{A}, h_{c,S}^{min}) & \tt= \Pr(\mathcal{A}=1|b=1)- \Pr(\mathcal{A}=1|b=0)\\
                                                                & \tt= \mathbb{E}[\mathcal{A}|b=1] - \mathbb{E}[\mathcal{A}|b=0] \\
                                                                &\tt = E[\mathcal{A}(h_{c,S}^{min})|b=1] - \mathbb{E}[\mathcal{A}(h_{c,S}^{min})|b=0]\\
                                                                &\tt = \mathbb{E}[\mathcal{A}(h_{c,S}^{min})] - \mathbb{E}[\mathcal{A}(h_{c,S}^{min})]
                          =0
\end{split}
\end{align}
where the second last equality follows since any function of $\tt h_{c,S}^{min}$ is independent of the training dataset.
\end{proof}

\section{Robustness to Attribute Inference Attacks}
\label{model-inversion}

\attribute*

\begin{proof}
    The proof follows trivially from definition of a causal model. $\tt h_c$ includes only causal features during training.  
    Thus, $\tt h(x)$ is independent of all features not in $\tt X_{c}$. 
    \begin{align}
      \tt  Adv(\mathcal{A}, h) & \tt = Pr(\mathcal{A}=1|x_s=1) - Pr(\mathcal{A}=1|x_s=0) \nonumber \\
                            & \tt = Pr(\mathcal{A}(h)=1|x_s=1) - Pr(\mathcal{A}(h)=1|x_s=0) \nonumber \\
                            &\tt = Pr(\mathcal{A}(h)=1) - Pr(\mathcal{A}(h)=1)=0 \nonumber
    \end{align}
\end{proof}

\section{Dataset Distribution}
\label{app:exp}
The target model is trained using the synthetic training and test data generated using the bnlearn library. We first divide the total dataset into training and test dataset in a 60:40 ratio. Further, the output of the trained model for each of the training and test dataset is again divided into 50:50 ratio. The training set for the attacker model consists of confidence values of the target model for the training as well as the test dataset. The relation is explained in Figure~\ref{fig:data_dist}.
Note that the attacker model is trained on the confidence output of the target models. 
\begin{figure}[h]
\centering
        \includegraphics[scale=0.35]{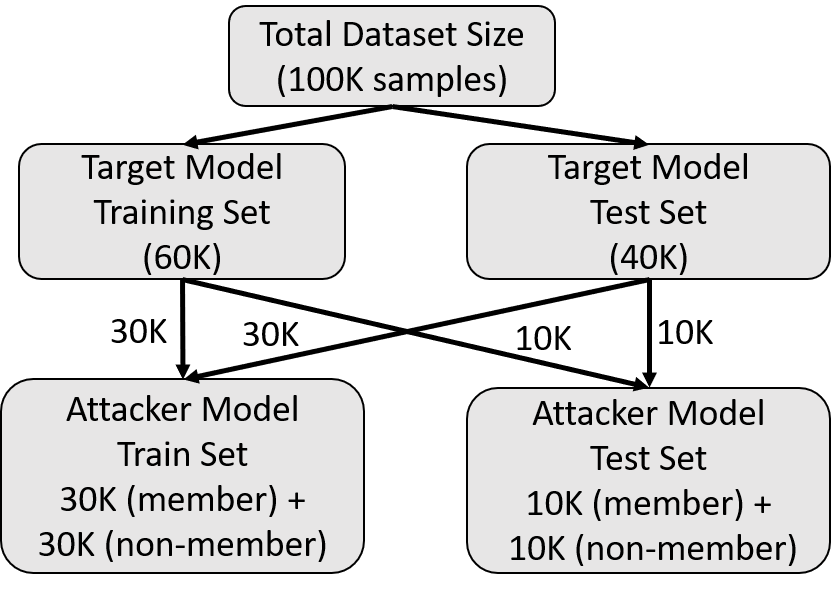}
        \caption{Dataset division for training target and attacker models.}
\label{fig:data_dist}
  \end{figure}

\end{document}